\documentclass{article}

\usepackage[final,nonatbib]{neurips_2024}

\usepackage[utf8]{inputenc} 
\usepackage[T1]{fontenc}    
\usepackage{hyperref}       
\usepackage{url}            
\usepackage{booktabs}       
\usepackage{amsfonts}       
\usepackage{nicefrac}       
\usepackage{microtype}      
\usepackage{xcolor}         

\ifdefined\XeTeXversion
  \usepackage{newunicodechar}
  \newunicodechar{–}{--}
  \newunicodechar{—}{---}
  \newunicodechar{“}{``}
  \newunicodechar{”}{''}
\fi

\definecolor{repaircolor}{HTML}{000000}
\DeclareRobustCommand{\repair}[1]{{\color{repaircolor}#1}}
\newenvironment{repairblock}{\begingroup\color{repaircolor}}{\endgroup}

\usepackage{amssymb}
\usepackage{amsmath,nccmath}
\usepackage{bbm}
\usepackage{multirow}
\usepackage{makecell}
\usepackage{amsthm}
\newtheorem{definition}{Definition}[section]
\newtheorem{theorem}{Theorem}[section]

\newtheorem{lemma}[theorem]{Lemma}

\usepackage{bigints}
\usepackage{graphicx}
\usepackage{enumitem}
\usepackage{tabu}

\usepackage{minitoc}

\newcommand{\colvec}[2][.8]{%
  \scalebox{#1}{%
    \renewcommand{\arraystretch}{.8}%
    $\begin{bmatrix}#2\end{bmatrix}$%
  }
}

\title{Identifiability Analysis of Linear ODE Systems with Hidden Confounders}

\author{%
  Yuanyuan~Wang\\
  The University of Melbourne\\
  \texttt{yuanyuanw2@student.unimelb.edu.au} \\
    \And
  Biwei~Huang \\
  University of California, San Diego\\
  \texttt{bih007@ucsd.edu} \\
    \And
  Wei~Huang \\
  The University of Melbourne\\
  \texttt{wei.huang@unimelb.edu.au} \\
    \And
  Xi~Geng \\
  The University of Melbourne\\
  \texttt{xi.geng@unimelb.edu.au} \\
    \And
  Mingming~Gong \thanks{Corresponding author.}\\
  The University of Melbourne\\
  \texttt{mingming.gong@unimelb.edu.au} \\
}

\begin{document}

\maketitle

\begin{abstract}
The identifiability analysis of linear Ordinary Differential Equation (ODE) systems is a necessary prerequisite for making reliable causal inferences about these systems. While identifiability has been well studied in scenarios where the system is fully observable, the conditions for identifiability remain unexplored when latent variables interact with the system. This paper aims to address this gap by presenting a systematic analysis of identifiability in linear ODE systems incorporating hidden confounders. Specifically, we investigate two cases of such systems. In the first case, latent confounders exhibit no causal relationships, yet their evolution adheres to specific functional forms, such as polynomial functions of time $t$. Subsequently, we extend this analysis to encompass scenarios where hidden confounders exhibit causal dependencies, with the causal structure of latent variables described by a Directed Acyclic Graph (DAG). The second case represents a more intricate variation of the first case, prompting a more comprehensive identifiability analysis. Accordingly, we conduct detailed identifiability analyses of the second system under various observation conditions, including both continuous and discrete observations from single or multiple trajectories. To validate our theoretical results, we perform a series of simulations, which support and substantiate our findings.
\end{abstract}
\section{Introduction}
Understanding the dynamics of systems governed by Ordinary Differential Equations (ODEs) is fundamental in various scientific disciplines, from physics \cite{bzhikhatlov2017research, koleva2010two, mandelzweig2001quasilinearization, zhong2019symplectic}, biology \cite{gratie2013ode, polynikis2009comparing, quach2007estimating, stadter2021benchmarking, su2021deep} to economics \cite{dockner2000differential, tsoularis2021some,  tu2012dynamical, weber2011optimal}. These ODE systems provide a natural framework for modeling causal relationships among system variables, enabling us to make reliable interpretations and interventions \cite{mooij2013ordinary, rubenstein2018deterministic, scholkopf2021toward}. Central to unraveling the causal mechanisms of such systems is the concept of identifiability analysis, which aims to uncover conditions under which system parameters can be uniquely determined from error-free observations. Identifiability is crucial for ensuring reliable parameter estimates, thereby guaranteeing reliable causal inferences about the system \cite{wang2024generator}. The motivation for our research on the identifiability analysis of ODE systems arises from the necessity of making reliable causal inferences about these systems.

Our research focuses on the homogeneous linear ODE system, represented as:
\begin{equation}\label{eq:ODE}
    \dot{\boldsymbol{x}}(t) = A \boldsymbol{x}(t)\,, \ \ \ 
    \boldsymbol{x}(0) = \boldsymbol{x}_0\,,
\end{equation}
where $t\in [0, \infty)$ denotes time, $\boldsymbol{x}(t)\in \mathbb{R}^d$ represents the system's state at time $t$, $\dot{\boldsymbol{x}}(t)$ denotes the first derivative of $\boldsymbol{x}(t)$ w.r.t. time, and $\boldsymbol{x}_0$ represents the initial condition of the system. The solution (trajectory) of the system, denoted as $\boldsymbol{x}(t; \boldsymbol{x}_0, A)$ for $t\in [0, \infty)$, is a single $d$-dimensional trajectory initialized with $\boldsymbol{x}_0$. 

Existing literature has extensively examined the identifiability of linear ODE systems under the assumption of complete observability, where all state variables are directly observable \cite{bellman1970structural, gargash1980necessary, glover1974parametrizations, grewal1976identifiability,  qiu2022identifiability, stanhope2014identifiability, wang2024identifiability}. Specifically, researchers have investigated identifiability of the ODE system \eqref{eq:ODE} from a single whole trajectory \cite{qiu2022identifiability, stanhope2014identifiability}, and extended analysis to discrete observations sampled from the trajectory \cite{wang2024identifiability}. However, practical scenarios often entail systems with latent variables, rendering them not entirely observable.  In this paper, we explore the identifiability analysis of this ODE system under latent confounders, particularly examining cases where no causal relationships exist from observable variables to latent variables, a commonly assumed condition in causality analysis with hidden variables \cite{cai2019triad,  chen2021fritl, huang2022latent, kaltenpoth2023causal, xie2020generalized, xie2022identification}.

In this paper, we focus on two scenarios:
\begin{enumerate}
    \item \textbf{Independent latent confounders:} Latent variables exhibit no causal relationships among themselves, leading to the following linear ODE system:\begin{equation}\label{eq:ODE1}
    \begin{bmatrix}
        \dot{\boldsymbol{x}}(t) \\
        \dot{\boldsymbol{z}}(t)
    \end{bmatrix}
    = \begin{bmatrix}
        A & B\\
        \boldsymbol{0} & \boldsymbol{0}
    \end{bmatrix}
    \begin{bmatrix}
        \boldsymbol{x}(t)\\
        \boldsymbol{z}(t)
    \end{bmatrix} +     \begin{bmatrix}
        \boldsymbol{0}\\
        \boldsymbol{f}(t)
    \end{bmatrix}\,, \ \ \
    \begin{bmatrix}
        \boldsymbol{x}(0)\\
        \boldsymbol{z}(0)
    \end{bmatrix}
    =\begin{bmatrix}
        \boldsymbol{x}_0\\
        \boldsymbol{z}_0
    \end{bmatrix}\,.
\end{equation}
    \item \textbf{Causally related latent confounders:} Latent variables exhibit causal relationships among themselves, specifically, they follow a DAG structure, represented as:
    \begin{equation}\label{eq:ODE2}
    \begin{bmatrix}
        \dot{\boldsymbol{x}}(t) \\
        \dot{\boldsymbol{z}}(t)
    \end{bmatrix}
    = \begin{bmatrix}
        A & B\\
        \boldsymbol{0} & G
    \end{bmatrix}
    \begin{bmatrix}
        \boldsymbol{x}(t)\\
        \boldsymbol{z}(t)
    \end{bmatrix}\,, \ \ \
    \begin{bmatrix}
        \boldsymbol{x}(0)\\
        \boldsymbol{z}(0)
    \end{bmatrix}
    =\begin{bmatrix}
        \boldsymbol{x}_0\\
        \boldsymbol{z}_0
    \end{bmatrix}\,.
\end{equation}
\end{enumerate}

In these two ODE systems, $\boldsymbol{x}(t) \in \mathbb{R}^d$ denotes the state of observable variables $\boldsymbol{x}= (x_1, x_2, \ldots, x_d)$, while $\boldsymbol{z}(t) \in \mathbb{R}^p$ denotes the state of latent variables $\boldsymbol{z} = (z_1, z_2, \ldots, z_p)$. Example causal structures of these two ODE systems are illustrated in Figure \ref{fig}. It is noteworthy that the structure may include cycles and self-loops within the observable variables. Additionally, two real-world examples are provided in Appendix \ref{app:real world examples}.
\begin{figure}[ht]
  \centering
  \includegraphics[width=0.95\textwidth]{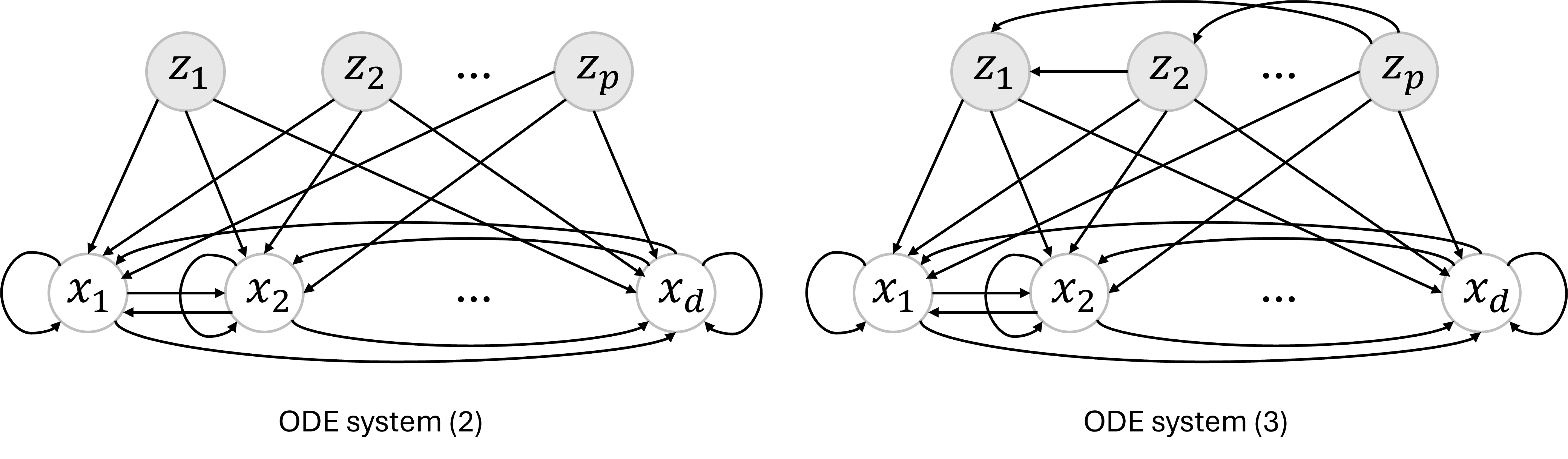}
  \caption{Example causal structures of the ODE system \eqref{eq:ODE1} and \eqref{eq:ODE2}.}
  \label{fig}
\end{figure}

This paper provides an identifiability analysis for the ODE system \eqref{eq:ODE1} under specific latent variable evolutions, such as polynomial functions of time $t$. Additionally, we conduct a systematic identifiability analysis of the ODE system \eqref{eq:ODE2} when the causal structure of the latent variables can be described by a DAG.

\section{Background}

\subsection{Causal interpretation of the ODE system}
When an ODE system describes the underlying causal mechanisms governing a dynamic system, it provides a natural framework for modeling causal relationships among system variables. The causal structure inherent in such systems can be directly read off \cite{mooij2013ordinary,scholkopf2021toward}. For instance, in the ODE system \eqref{eq:ODE}, where the $ij$-th entry of the parameter matrix $A$ is denoted as $A_{ij}$, the presence of $A_{ij} \neq 0$ signifies that the derivative of $x_i(t)$ is influenced by $x_j(t)$, thus indicating a causal link from $x_j$ to $x_i$. Here, $x_i$ denotes the $i$-th variable of the ODE system \eqref{eq:ODE}, and $x_i(t)$ represents its state at time $t$. Since the right hand side of the ODE system \eqref{eq:ODE} does not explicitly depend on time $t$, the causal structure of this ODE system is time-invariant.

An essential prerequisite for reliably inferring the causal structure and effects of an ODE system, for purposes of interpretation or intervention, is the identifiability analysis of such systems. To underscore this necessity, we provide an illustrative example. Consider the ODE system \eqref{eq:ODE2}. Set
\begin{equation*}
    \boldsymbol{x}_0 = \begin{bmatrix}
     1\\
     1
    \end{bmatrix}\,,\ \ 
    \boldsymbol{z}_0 = \begin{bmatrix}
     1\\
     1
    \end{bmatrix}\,, \ \ 
        B = \begin{bmatrix}
        1 & 1\\
        1 & 1
    \end{bmatrix}\,,\ \ 
    G = \begin{bmatrix}
        0 & 1\\
        0 & 0
    \end{bmatrix}\,,
\end{equation*}
\begin{equation*}
        A = \begin{bmatrix}
        1 & 0\\
        0 & 1
    \end{bmatrix}\,,\ \ \ 
    A' = \begin{bmatrix}
    0 & 1\\
    1 & 0
\end{bmatrix}\,, \ \ \ 
    M = \begin{bmatrix}
    A & B\\
    \boldsymbol{0} & G
    \end{bmatrix}\,, \ \ \ 
    M' = \begin{bmatrix}
    A' & B\\
    \boldsymbol{0} & G
    \end{bmatrix}\,.
\end{equation*}
Calculations reveal that the solutions (trajectory) of the ODE system \eqref{eq:ODE2} with parameter matrices $M$ or $M'$ are identical, i.e.,
\begin{equation*}
\begin{bmatrix}
    \boldsymbol{x}(t)\\
    \boldsymbol{z}(t)
\end{bmatrix}
     = e^{Mt}\begin{bmatrix}
    \boldsymbol{x}_0\\
    \boldsymbol{z}_0
\end{bmatrix}
     = e^{M't}\begin{bmatrix}
    \boldsymbol{x}_0\\
    \boldsymbol{z}_0
\end{bmatrix}\,.
\end{equation*}
This indicates that using observations sampled from this trajectory to estimate parameter matrix $M$ may end up yielding $M'$, which exhibits a fundamentally distinct causal relationship between $x_1$ and $x_2$, see Figure \ref{fig:example}. 
\begin{figure}[ht]
  \centering
  \includegraphics[width=0.6\textwidth]{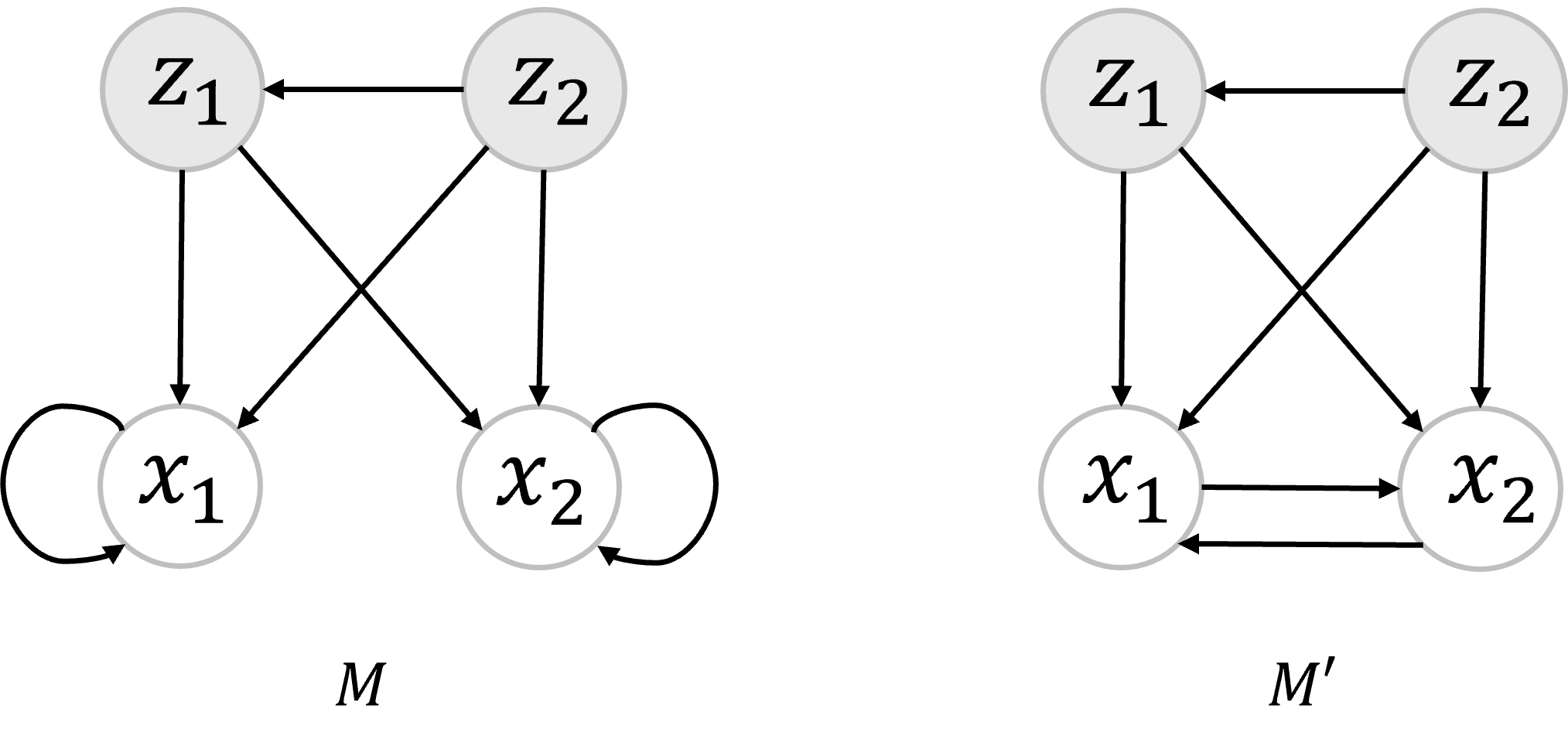}
  \caption{Causal structures of the ODE system \eqref{eq:ODE2} with parameter matrix $M$ and $M'$.}
  \label{fig:example}
\end{figure}
This discrepancy in parameter estimation, wherein $M'$ is obtained instead of the true underlying parameter matrix $M$, may lead to misleading interpretations and causal inferences, potentially influencing decision-making, particularly regarding interventions. For instance, intervention with $x_1(t) = 1$, under the true underlying parameter matrix $M$, yields the trajectory $x_2(t) = 4e^t-t-3$ (post-intervention), whereas under matrix $M'$, the trajectory becomes $x_2(t) = t^2/2 + 3t +1$ (post-intervention). Detailed calculations are provided in Appendix \ref{app:unidentifiable example}.

\subsection{Identifiability analysis of the linear ODE system (\ref{eq:ODE})}
The identifiability analysis of the ODE system \eqref{eq:ODE} has been well studied. Here, we present a fundamental definition and theorem essential for understanding identifiability in the ODE system \eqref{eq:ODE}. Denoting its solution as $\boldsymbol{x}(t; \boldsymbol{x}_0, A)$, it is noteworthy that the system is fully observable, without latent variables interacting with it. We present the identifiability definition and theorem as follows.

\begin{definition}\label{def:identifiability ODE}
For $\boldsymbol{x}_0 \in \mathbb{R}^d, A\in \mathbb{R}^{d\times d}$, the ODE system \eqref{eq:ODE} is said to be $(\boldsymbol{x}_0, A)$-identifiable, if for all $\boldsymbol{x}'_0 \in \mathbb{R}^d$ and all $A'\in \mathbb{R}^{d\times d}$, with
$(\boldsymbol{x}_0, A) \neq (\boldsymbol{x}'_0, A')$, it holds that $\boldsymbol{x}(\cdot; \boldsymbol{x}_0, A)\ {\neq}\boldsymbol{x}(\cdot; \boldsymbol{x}'_0, A') $.\footnote{$\boldsymbol{x}(\cdot; \boldsymbol{x}_0, A) = \{\boldsymbol{x}(t; \boldsymbol{x}_0, A): 0 \leqslant t < \infty\}$, this inequation means that there exists at least one $t\geqslant 0$ such that $\boldsymbol{x}(t; \boldsymbol{x}_0, A) \neq \boldsymbol{x}(t; \boldsymbol{x}'_0, A') $.} 
\end{definition}

\begin{lemma}\label{lemma:identifiability ODE}
For $\boldsymbol{x}_0 \in \mathbb{R}^d, A\in \mathbb{R}^{d\times d}$, the ODE system \eqref{eq:ODE} is $(\boldsymbol{x}_0, A)$-identifiable if and only if condition \textbf{A0} is satisfied.
\begin{enumerate}
    \item [\textbf{A0}] the set of vectors $\{\boldsymbol{x}_0, A\boldsymbol{x}_0, \ldots, A^{d-1}\boldsymbol{x}_0\}$ is linearly independent.
\end{enumerate}
\end{lemma}

Definition \ref{def:identifiability ODE} and Theorem \ref{lemma:identifiability ODE} are adapted from \cite[Definition 1]{wang2024identifiability} and \cite[Lemma2]{wang2024identifiability}. We use $\boldsymbol{x}_0'$ and $A'$ to distinguish other system parameters from the true system parameters $\boldsymbol{x}_0$ and $A$; $\boldsymbol{x}_0'$ and $A'$ can represent any $d$-dimensional initial conditions and any $d\times d$ parameter matrices, respectively. Here instead of describing a collective property of a set of systems, we describe an intrinsic property of a single system with parameters ($\boldsymbol{x}_0, A$). In practice, the aim is to ascertain whether the true underlying system parameter ($\boldsymbol{x}_0, A$) is uniquely determined by observations. Hence, ($\boldsymbol{x}_0, A$)-identifiability offers a more intuitive description of the identifiability of the ODE system from a practical perspective.

From a geometric perspective, condition \textbf{A0} stated in Lemma \ref{lemma:identifiability ODE} indicates that the initial condition $\boldsymbol{x}_0$ is not contained in an $A$-invariant \textbf{proper} subspace of $\mathbb{R}^d$. Intuitively, this means the trajectory of this system started from $\boldsymbol{x}_0$ spans the entire $d$-dimensional state space. That is, our observations cover information on all dimensions of the state space, thus rendering the identifiability of the system. Additionally, condition \textbf{A0} is generic, as noted in \cite{wang2024generator}, meaning that the set of system parameters violating this condition has Lebesgue measure zero. Thus, condition \textbf{A0} is satisfied for almost all combinations of $\boldsymbol{x}_0$ and $A$.
\section{Identifiability analysis of the linear ODE system (\ref{eq:ODE1})}\label{sec:ODE1}
In this section, we present the identifiability condition for the linear ODE system \eqref{eq:ODE1}.
We consider the function $\boldsymbol{f}(t)$ in \eqref{eq:ODE1} as a specific function of time $t$. Here we first define $\boldsymbol{f}(t)$ as a $r$-degree polynomial function of time $t$, expressed as follows: 
\begin{equation}\label{eq:f(t)}
    \boldsymbol{f}(t) = \sum_{k=0}^r \boldsymbol{v}_k t^k\,, \ \ \ \boldsymbol{v}_k \in \mathbb{R}^{p}\,.
\end{equation}
Simple calculations show that 
\begin{equation*}
    \boldsymbol{z}(t) =\sum_{k=0}^r \cfrac{\boldsymbol{v}_k}{k+1} t^{k+1} + \boldsymbol{z}_0\,.
\end{equation*}
Thus, 
\begin{equation}\label{eq:dxt1}
    \dot{\boldsymbol{x}}(t) = A \boldsymbol{x}(t) + B \boldsymbol{z}(t)\\
    = A \boldsymbol{x}(t) + \sum_{k=0}^{r} \cfrac{B\boldsymbol{v}_k}{k+1} t^{k+1} + B\boldsymbol{z}_0\,.
\end{equation}

We denote the unknown parameters of the ODE system \eqref{eq:ODE1} as $\boldsymbol{\theta}$, specifically, $\boldsymbol{\theta}:=(\boldsymbol{x}_0,\boldsymbol{z}_0, A, B, \{\boldsymbol{v}_k\}_0^r)$, where $\{\boldsymbol{v}_k\}_0^r$ denotes all the $\boldsymbol{v}_k$'s for $k = 0, \ldots, r$. Let $[\boldsymbol{x}^T(t; \boldsymbol{\theta}), \boldsymbol{z}^T(t; \boldsymbol{\theta})]^{T}$ denote the solution of the ODE system \eqref{eq:ODE1}. It is important to note that under our hidden variables setting, only $\boldsymbol{x}(t; \boldsymbol{\theta})$ is observable. Based on Equation \eqref{eq:dxt1}, we present the following identifiability definition.

\begin{definition}\label{def:identifiability ODE1}
For $\boldsymbol{x}_0 \in \mathbb{R}^d, \boldsymbol{z}_0 \in \mathbb{R}^p, A\in \mathbb{R}^{d\times d}, B\in \mathbb{R}^{d\times p}$ and $\{\boldsymbol{v}_k\}_0^r \in \mathbb{R}^{p}$, for all $\boldsymbol{x}'_0 \in \mathbb{R}^d$, all $\boldsymbol{z}'_0 \in \mathbb{R}^p$, all $A'\in \mathbb{R}^{d\times d}$, all $B' \in \mathbb{R}^{d\times p}$, and all $\{\boldsymbol{v}'_k\}_0^r \in \mathbb{R}^{p}$, we denote $\boldsymbol{\theta}':=(\boldsymbol{x}'_0,\boldsymbol{z}'_0, A', B', \{\boldsymbol{v}'_k\}_1^r)$, we say the ODE system \eqref{eq:ODE1} is
 $\boldsymbol{\theta}$-identifiable:
if $(\boldsymbol{x}_0, A, B\boldsymbol{z}_0, \{B\boldsymbol{v}_k\}_0^r) \neq (\boldsymbol{x}'_0, A', B'\boldsymbol{z}'_0, \{B'\boldsymbol{v}_k'\}_0^r)$, it holds that $\boldsymbol{x}(\cdot; \boldsymbol{\theta})\ {\neq} \ \boldsymbol{x}(\cdot; \boldsymbol{\theta}') $.
\end{definition}

In the ODE system \eqref{eq:ODE1}, where only variables $\boldsymbol{x}$ are observable, we will, with some terminological leniency, refer to $\boldsymbol{x}(\cdot; \boldsymbol{\theta})$ as the trajectory of the ODE system \eqref{eq:ODE1} with parameters $\boldsymbol{\theta}$. According to Definition \ref{def:identifiability ODE1}, if the ODE system \eqref{eq:ODE1} with a polynomial $\boldsymbol{f}(t)$ is $\boldsymbol{\theta}$-identifiable, then the trajectory of the system can uniquely determine the values of $(\boldsymbol{x}_0, A, B\boldsymbol{z}_0, \{B\boldsymbol{v}_k\}_0^r)$. This determination is sufficient to identify the causal relationships between observable variables $\boldsymbol{x}$ as described by Equation \eqref{eq:dxt1}. Consequently, one can safely intervene in the observable variables of the ODE system and make reliable causal inferences, despite the fact that matrix $B$ cannot be identified under this definition.

\begin{theorem}\label{thm:identifiability ODE1}
    For $\boldsymbol{x}_0 \in \mathbb{R}^d, \boldsymbol{z}_0 \in \mathbb{R}^p, A\in \mathbb{R}^{d\times d}, B\in \mathbb{R}^{d\times p}$, $\{\boldsymbol{v}_k\}_0^r \in \mathbb{R}^{p}$, the ODE system \eqref{eq:ODE1} is $\boldsymbol{\theta}$-identifiable if and only if assumption \textbf{A1} is satisfied.
    \begin{enumerate}
        \item [\textbf{A1}] the set of vectors $\{\boldsymbol{\beta}, A\boldsymbol{\beta}, \ldots, A^{d-1}\boldsymbol{\beta} \}$ is linearly independent, where $ \boldsymbol{\beta} = A^{r+1}(A\boldsymbol{x}_0 + B\boldsymbol{z}_0) + \sum_{j=0}^{r}j!A^{r-j} B \boldsymbol{v}_j$, and $j!$ denotes the factorial of $j$.
    \end{enumerate}
\end{theorem}

The proof of Theorem \ref{thm:identifiability ODE1} can be found in Appendix \ref{proof: ODE1}. Condition \textbf{A1} is both sufficient and necessary, indicating, from a geometric perspective, that the vector $\boldsymbol{\beta}$ is not contained in an $A$-invariant proper subspace of $\mathbb{R}^d$ \cite[Lemma 3.1]{stanhope2014identifiability}.

The key point of the proof is the introduction of an augmented state $\boldsymbol{y}(t) = [\boldsymbol{x}^T(t), 1, t, t^2, \dots, t^{r+1}]^{T}$ with a corresponding ODE system:
\begin{equation}\label{eq:new ODE y}
\begin{split}
    \dot{\boldsymbol{y}}(t) &= \underbrace{\begin{bmatrix}
        A & B\boldsymbol{z}_0 & B\boldsymbol{v}_0 & \ldots & B\boldsymbol{v}_{r-1}/r & B\boldsymbol{v}_r/(r+1)\\
        \boldsymbol{0}_d & 0 & 0 & \ldots & 0 & 0\\
        \boldsymbol{0}_d & 1 & 0 & \ldots & 0 & 0\\
        \vdots & \vdots & \vdots & \ddots & \vdots & \vdots \\
        \boldsymbol{0}_d & 0 & 0 & \ldots & r+1 & 0
    \end{bmatrix}}_\textrm{denoted as $F$} \boldsymbol{y}(t)  \,,  \\
    \boldsymbol{y}(0) &= [\boldsymbol{x}_0^T, 1, 0, \ldots, 0]^T := \boldsymbol{y}_0\,,
\end{split}
\end{equation}
where $\boldsymbol{0}_d$ is a $d$-dimensional zero row vector, and matrix $F \in \mathbb{R}^{(d+r+2)\times (d+r+2)}$. The ODE system \eqref{eq:new ODE y} is a homogeneous linear ODE system analogous to \eqref{eq:ODE} but with fully observable variables $\boldsymbol{y}$. In other words, we transform our system of interest, \eqref{eq:ODE1},  which includes hidden confounders, into a fully observable ODE system \eqref{eq:new ODE y}. This allows us to leverage existing identifiability results for homogeneous linear ODE systems, specifically Lemma \ref{lemma:identifiability ODE}, to derive the identifiability condition for the ODE system \eqref{eq:ODE1}. 

Based on this approach, if the state of the hidden variables $\boldsymbol{z}(t)$, as determined by the function $\boldsymbol{f}(t)$ in the ODE system \eqref{eq:ODE1}, can be described by some linear combinations of observable functions of time $t$, then the identifiability condition of the ODE system \eqref{eq:ODE1} can be derived. For an illustration, in the Appendix \ref{app:more cases of ODE1}, we provide identifiability conditions for the ODE system \eqref{eq:ODE1} when $\boldsymbol{f}(t) = \boldsymbol{v}e^t$ and $\boldsymbol{f}(t) = \boldsymbol{v}_1 sin(t) + \boldsymbol{v}_2 cos(t)$. While we do not enumerate all functions $\boldsymbol{f}(t)$ that meet this condition, our primary objective is to demonstrate a method for deriving the identifiability condition for the ODE \eqref{eq:ODE1} when the evolution of its hidden variables conforms to certain specific functions. Researchers can apply this approach to find appropriate functions $\boldsymbol{f}(t)$ according to their specific requirements.
\section{Identifiability analysis of the linear ODE system (\ref{eq:ODE2})}
In this section, we extend the identifiability analysis to linear ODE systems with causally related latent confounders. Specifically, we assume that the causal structure of latent variables satisfies the following latent DAG assumption.


\indent\hspace{0.6cm}\textit{\textbf{Latent DAG}: the causal structure of latent variables can be described by a DAG. }

The DAG assumption is commonly employed in causality studies \cite{cai2019triad, chen2021fritl, kaltenpoth2023causal, pearl2009causality, verma1990causal, xie2020generalized}. Under the latent DAG assumption, the matrix $G$ can be permuted to be a strictly upper triangular matrix, i.e., an upper triangular matrix with zeros along the main diagonal \cite{chen2021fritl, hoyer2008estimation}. Without loss of generality, we set $G$ as a strictly upper triangular matrix.

Since $G$ is a strictly upper triangular matrix, by the Cayley–Hamilton theorem \cite{straubing1983combinatorial}, $G$ is a nilpotent matrix with an index $\leqslant p$. Consequently, $G^k = 0$ for all $k\geqslant p$.
 
Based on \cite{stanhope2014identifiability, teschl2012ordinary}, the solution of $\boldsymbol{z}(t)$ can be expressed as:
\begin{equation*}
    \boldsymbol{z}(t) = e^{Gt}\boldsymbol{z}_0
    = \sum_{k=0}^\infty \cfrac{G^k\boldsymbol{z}_0}{k!}t^k
    = \sum_{k=0}^{p-1} \cfrac{G^k\boldsymbol{z}_0}{k!}t^k\,.
\end{equation*}
Thus,
\begin{equation}\label{eq:dxt}
    \dot{\boldsymbol{x}}(t) = A \boldsymbol{x}(t) + B \boldsymbol{z}(t)\\
    = A \boldsymbol{x}(t) + \sum_{k=0}^{p-1} \cfrac{BG^k\boldsymbol{z}_0}{k!}t^k\,.
\end{equation}
We observe that Equation \eqref{eq:dxt} has the same function form as Equation \eqref{eq:dxt1}, but with different coefficients (system parameters) for the polynomial of time $t$. Therefore, the ODE system \eqref{eq:ODE2} under the latent DAG assumption can be considered a more complex version of the ODE system \eqref{eq:ODE1} when $\boldsymbol{f}(t)$ follows a polynomial function of time $t$. Since the ODE system \eqref{eq:ODE2} incorporates causally related latent confounders, which is a more interesting and practical case, we will provide a more comprehensive identifiability analysis of the ODE system \eqref{eq:ODE2}. The derived identifiability results can be easily generated to the case of the ODE system \eqref{eq:ODE1}.

\subsection{Identifiability condition from a single whole trajectory}

We denote the unknown parameters of the ODE system \eqref{eq:ODE2} as $\boldsymbol{\eta}$, that is, $\boldsymbol{\eta}:=(\boldsymbol{x}_0,\boldsymbol{z}_0, A, B, G)$. We further denote the solution of the ODE system \eqref{eq:ODE2} as $[\boldsymbol{x}^T(t; \boldsymbol{\eta}), \boldsymbol{z}^T(t; \boldsymbol{\eta})]^{T}$; note that under our latent variables setting, only $\boldsymbol{x}(t; \boldsymbol{\eta})$ is observable. Thus, based on Equation \eqref{eq:dxt}, we present the following identifiability definition.

\begin{definition}\label{def: identifiability ODE2 continuous}
For $\boldsymbol{x}_0 \in \mathbb{R}^d, \boldsymbol{z}_0 \in \mathbb{R}^p, A\in \mathbb{R}^{d\times d}, B\in \mathbb{R}^{d\times p}$ and $G \in \mathbb{R}^{p\times p}$, under the latent DAG assumption, for all $\boldsymbol{x}'_0 \in \mathbb{R}^d$, all $\boldsymbol{z}'_0 \in \mathbb{R}^p$, all $A'\in \mathbb{R}^{d\times d}$, all $B' \in \mathbb{R}^{d\times p}$, and all $G' \in \mathbb{R}^{p\times p}$, we denote $\boldsymbol{\eta}' := (\boldsymbol{x}'_0,\boldsymbol{z}'_0, A', B', G') $, we say the ODE system \eqref{eq:ODE2} is $\boldsymbol{\eta}$-identifiable:
if $(\boldsymbol{x}_0, A, B\boldsymbol{z}_0, BG\boldsymbol{z}_0, \ldots, BG^{p-1}\boldsymbol{z}_0) \neq (\boldsymbol{x}'_0, A', B'\boldsymbol{z}'_0, B'G'\boldsymbol{z}'_0, \ldots, B'G'^{p-1}\boldsymbol{z}'_0)$, it holds that $\boldsymbol{x}(\cdot; \boldsymbol{\eta})\ $${\neq}$ $\  \boldsymbol{x}(\cdot; \boldsymbol{\eta}') $.
\end{definition}

Similar to the case of the ODE system \eqref{eq:ODE1}, we refer to $\boldsymbol{x}(\cdot;\boldsymbol{\eta})$ as the trajectory of the ODE system \eqref{eq:ODE2} with parameters $\boldsymbol{\eta}$. Definition \ref{def: identifiability ODE2 continuous} defines the identifiability of the ODE system \eqref{eq:ODE2} from a single whole trajectory $\boldsymbol{x}(\cdot;\boldsymbol{\eta})$.
Once the ODE system \eqref{eq:ODE2} is $\boldsymbol{\eta}$-identifiable, the causal relationships among the observable variables $\boldsymbol{x}$ can be determined through Equation \eqref{eq:dxt}. We then establish the condition for the identifiability of the ODE system \eqref{eq:ODE2} based on Definition \ref{def: identifiability ODE2 continuous}.

\begin{theorem}\label{thm: identifiability ODE2 continuous}
    For $\boldsymbol{x}_0 \in \mathbb{R}^d, \boldsymbol{z}_0 \in \mathbb{R}^p, A\in \mathbb{R}^{d\times d}, B\in \mathbb{R}^{d\times p}$ and $G \in \mathbb{R}^{p\times p}$, under the latent DAG assumption, the ODE system \eqref{eq:ODE2} is $\boldsymbol{\eta}$-identifiable if and only if assumption \textbf{B1} is satisfied.
\begin{enumerate}
    \item [\textbf{B1}:] the set of vectors $\{\boldsymbol{\gamma}, A\boldsymbol{\gamma}, \ldots, A^{d-1}\boldsymbol{\gamma} \}$ is linearly independent, where $ \boldsymbol{\gamma} = A^{p}\boldsymbol{x}_0 + \sum_{j=0}^{p-1}A^{p-1-j} B G^j \boldsymbol{z}_0$.
\end{enumerate}
\end{theorem}

The proof of Theorem \ref{thm: identifiability ODE2 continuous} can be found in Appendix \ref{proof: ODE2 continuous}. Condition \textbf{B1} is both sufficient and necessary, and from a geometric perspective, it indicates that the vector $\boldsymbol{\gamma}$ is not contained in an $A$-invariant proper subspace of $\mathbb{R}^d$ \cite[Lemma 3.1]{stanhope2014identifiability}.

\subsection{Identifiability condition from discrete observations sampled from a single trajectory}

In practice, we often have access only to a sequence of discrete observations sampled from a trajectory rather than knowing the whole trajectory. Therefore, we also derive the identifiability conditions under the scenario where only discrete observations from a trajectory are available. Firstly, we extend the identifiability definition of the ODE system \eqref{eq:ODE2} as follows.

\begin{repairblock}
For the discrete-observation definitions in this section and Appendix~\ref{app:alternative approach}, fix a sampling interval $\Delta>0$ (see \textbf{C0} below) and let
\[
\mathcal A_\Delta=\{A\in\mathbb R^{d\times d}:|\operatorname{Im}\lambda|<\pi/\Delta
\text{ for every }\lambda\in\sigma(A)\},
\]
where $\sigma(A)$ denotes the spectrum of $A$, that is, the set of its (possibly complex) eigenvalues, and $\operatorname{Im}\lambda$ is the imaginary part of $\lambda$. Identifiability is understood within this class: both the true matrix $A$ and every competing matrix $A'$ are restricted to $\mathcal A_\Delta$.\footnote{\repair{See Appendix~\ref{app:discrete-correction} for a summary of the corrections to the discrete-observation results.}} Intuitively, it bounds oscillation frequencies relative to the sampling rate, leaves eigenvalue real parts unrestricted, and allows all matrices with real spectra.
\end{repairblock}

\begin{definition}\label{def: identifiability ODE2 discrete}
For $\boldsymbol{x}_0 \in \mathbb{R}^d, \boldsymbol{z}_0 \in \mathbb{R}^p, A\in \repair{\mathcal A_\Delta}, B\in \mathbb{R}^{d\times p}$ and $G \in \mathbb{R}^{p\times p}$. For any $n \geqslant 1$, let $t_j, j = 1, \ldots, n$ be any $n$ time points and $\boldsymbol{x}_j : = \boldsymbol{x}(t_j; \boldsymbol{\eta})$ be the error-free observation of the trajectory $\boldsymbol{x}(\cdot; \boldsymbol{\eta})$ at time $t_j$. Under the latent DAG assumption, we say the ODE system \eqref{eq:ODE2} is $\boldsymbol{\eta}$-identifiable from $\boldsymbol{x}_1, \ldots, \boldsymbol{x}_n$, if for all $\boldsymbol{x}'_0 \in \mathbb{R}^d$, all $\boldsymbol{z}'_0 \in \mathbb{R}^p$, all $A'\in \repair{\mathcal A_\Delta}$, all $B' \in \mathbb{R}^{d\times p}$, and all $G' \in \mathbb{R}^{p\times p}$ with $(\boldsymbol{x}_0, A, B\boldsymbol{z}_0, BG\boldsymbol{z}_0, \ldots, BG^{p-1}\boldsymbol{z}_0) \neq (\boldsymbol{x}'_0, A', B'\boldsymbol{z}'_0, B'G'\boldsymbol{z}'_0, \ldots, B'G'^{p-1}\boldsymbol{z}'_0)$, it holds that $\exists j \in \{1, \ldots, n\}$ such that $\boldsymbol{x}(t_j; \boldsymbol{\eta}){\neq} \boldsymbol{x}(t_j; \boldsymbol{\eta}')$.
\end{definition}

Definition \ref{def: identifiability ODE2 discrete} defines the identifiability of the ODE system \eqref{eq:ODE2} from $n$ observations sampled from the trajectory $\boldsymbol{x}(\cdot;\boldsymbol{\eta})$. Then we establish the condition for the identifiability of the ODE system \eqref{eq:ODE2} from discrete observations based on Definition \ref{def: identifiability ODE2 discrete}.

\begin{repairblock}
\noindent\textbf{C0 (sampling):}\label{assump:discrete_sampling}
$t_j=t_1+(j-1)\Delta$ for $j=1,\ldots,n$, where $t_1\ge0$ and $n\ge d+p+1$.
\end{repairblock}

\begin{theorem}\label{thm: identifiability ODE2 discrete}
For $\boldsymbol{x}_0 \in \mathbb{R}^d, \boldsymbol{z}_0 \in \mathbb{R}^p, A\in \repair{\mathcal A_\Delta}, B\in \mathbb{R}^{d\times p}$ and $G \in \mathbb{R}^{p\times p}$. We define new observation $\boldsymbol{y}_j := [\boldsymbol{x}_j^T, 1, t_j, t_j^2,\ldots, t_j^{p-1}]^T\in \mathbb{R}^{d+p}$, for $j = 1, \ldots, n$. Under the latent DAG assumption \repair{and \textbf{C0}}, the ODE system \eqref{eq:ODE2} is $\boldsymbol{\eta}$-identifiable from discrete observations $\boldsymbol{x}_1, \ldots, \boldsymbol{x}_n$, if and only if assumption \textbf{C1} is satisfied.
\begin{itemize}
    \item [\textbf{C1}:] there exists $(d+p)$ $\boldsymbol{y}_j$'s with indices denoting as $\{j_1, j_2, \ldots, j_{d+p}\} \subseteq \{1, 2, \ldots, n\}$, such that the set of vectors $\{\boldsymbol{y}_{j_1}, \boldsymbol{y}_{j_2}, \ldots, \boldsymbol{y}_{j_{d+p}}\}$ is linearly independent.
\end{itemize}
\end{theorem}

The proof of Theorem \ref{thm: identifiability ODE2 discrete} can be found in Appendix \ref{proof: ODE2 discrete}.  Condition \textbf{C1} is both sufficient and necessary. This theorem states that\repair{, under \textbf{C0},} as long as there are $d+p$ observations $\boldsymbol{x}_j$'s such that the corresponding augmented new observations $\boldsymbol{y}_j$'s are linearly independent, the ODE system \eqref{eq:ODE2} is $\boldsymbol{\eta}$-identifiable from \repair{the discrete observations $\boldsymbol{x}_1,\ldots,\boldsymbol{x}_n$}. Under the latent DAG assumption, we can transfer the ODE system \eqref{eq:ODE2}, which includes hidden confounders, into a $(d+p)$-dimensional fully observable ODE system \eqref{eq:ODE} through the augmented state $\boldsymbol{y}(t)$. Condition \textbf{C1} indicates that our observations span the entire $(d+p)$-dimensional state space, thus rendering the system identifiable.

Both Definition \ref{def: identifiability ODE2 continuous} and Definition \ref{def: identifiability ODE2 discrete} define the identifiability of the ODE system \eqref{eq:ODE2} to some extent of the unknown parameters. In other words, given the available observations, under Definition \ref{def: identifiability ODE2 continuous} and Definition \ref{def: identifiability ODE2 discrete}, one can only identify the values of $(\boldsymbol{x}_0, A, B\boldsymbol{z}_0, BG\boldsymbol{z}_0, \ldots, BG^{p-1}\boldsymbol{z}_0)$, but not the values of $(\boldsymbol{z}_0, B, G)$. Based on Equation \eqref{eq:dxt}, this level of identifiability is sufficient to identify the causal relationships between observable variables $\boldsymbol{x}$, enabling safe intervention on the observable variables with reliable causal inferences. However, in scenarios where practitioners can intervene in the latent variables and require inferring the causal effects of the intervened system, identifying the matrices $B$ and $G$ becomes essential for reliable causal references. For instance, in chemical kinetics, where the evolution of chemical concentrations over time can often be modeled by an ODE system \cite{burnham2008inference, craciun2005multiple}, some chemicals may not be measurable during the reaction, rendering them latent variables. Nonetheless, practitioners can intervene in these latent variables by setting specific initial concentrations. Therefore, we provide an identifiability analysis of the linear ODE system \eqref{eq:ODE2} when practitioners can control the initial condition of the latent variables: $\boldsymbol{z}_0$.

\subsection{Identifiability condition from \textit{p} controllable whole trajectories}\label{subsec:iden p continous}
Assuming the initial condition of the latent variables $\boldsymbol{z}_0$ is controllable, which means that the values of $\boldsymbol{z}_0$ can be treated as given values, we denote it as $\boldsymbol{z}_0^*$. 
In the following, we provide the identifiability condition of the ODE system \eqref{eq:ODE2} when we are given $p$ initial conditions $\boldsymbol{z}_0^*$, denoting as $\boldsymbol{z}_0^{*i}$. We first present the definition.
\begin{definition}\label{def: identifiability ODE2 more}
Given $\boldsymbol{z}_0^{*i} \in \mathbb{R}^p$ for $i=1, \ldots, p$, for $\boldsymbol{x}_0 \in \mathbb{R}^d,  A\in \mathbb{R}^{d\times d}, B\in \mathbb{R}^{d\times p}$ and $G \in \mathbb{R}^{p\times p}$, under the latent DAG assumption, for all $\boldsymbol{x}'_0 \in \mathbb{R}^d$, all $A'\in \mathbb{R}^{d\times d}$, all $B' \in \mathbb{R}^{d\times p}$, and all $G' \in \mathbb{R}^{p\times p}$, we denote $\boldsymbol{\eta}_i := (\boldsymbol{x}_0,\boldsymbol{z}_0^{*i}, A, B, G) $ and $\boldsymbol{\eta}'_i := (\boldsymbol{x}'_0,\boldsymbol{z}_0^{*i}, A', B', G') $, we say the ODE system \eqref{eq:ODE2} is $\{\boldsymbol{\eta}_i\}_1^p$-identifiable:
if $(\boldsymbol{x}_0, A, B, G) \neq (\boldsymbol{x}', A', B', G')$, it holds that $\exists i$ such that $\boldsymbol{x}(\cdot; \boldsymbol{\eta}_i)\ $${\neq}$ $\  \boldsymbol{x}(\cdot; \boldsymbol{\eta}'_i) $.
\end{definition}

Definition \ref{def: identifiability ODE2 more} defines the identifiability of the ODE system \eqref{eq:ODE2} from $p$ whole trajectories $\boldsymbol{x}(\cdot;\boldsymbol{\eta}_i)$ with $i=1, \ldots, p$, and under this definition, matrix $B$ and $G$ are also identifiable. Based on this definition, we provide the identifiability condition.

\begin{theorem}\label{thm: identifiability ODE2 more}
Given $\boldsymbol{z}_0^{*i} \in \mathbb{R}^p$ for $i=1, \ldots, p$, for $\boldsymbol{x}_0 \in \mathbb{R}^d,  A\in \mathbb{R}^{d\times d}, B\in \mathbb{R}^{d\times p}$ and $G \in \mathbb{R}^{p\times p}$, under the latent DAG assumption, the ODE system \eqref{eq:ODE2} is $\{\boldsymbol{\eta}_i\}_1^p$-identifiable if assumptions $\textbf{B}_2$, $\textbf{B}_3$ and $\textbf{B}_4$ are all satisfied.
\begin{itemize}
    \item [\textbf{B2}:] each $\boldsymbol{z}_0^{*i}$ for $i=1, \ldots, p$, satisfies assumption \textbf{B1}. That is, if we set $ \boldsymbol{\gamma}_i = A^{p}\boldsymbol{x}_0 + \sum_{j=0}^{p-1}A^{p-1-j} B G^j \boldsymbol{z}_0^{*i}$, then the set of vectors $\{\boldsymbol{\gamma}_i, A\boldsymbol{\gamma}_i, \ldots, A^{d-1}\boldsymbol{\gamma}_i \}$ is linearly independent for all $i=1, \ldots, p$.
    \item [\textbf{B3}:] the set of vectors $\{\boldsymbol{z}_0^{*1}, \boldsymbol{z}_0^{*2}, \ldots, \boldsymbol{z}_0^{*p} \}$ is linearly independent.
    \item [\textbf{B4}:] the matrix composed by vertically stack the matrices $\{B, BG, \ldots, BG^{p-1}\}$ has rank $p$.
\end{itemize}
\end{theorem}
The proof of Theorem \ref{thm: identifiability ODE2 more} can be found in Appendix \ref{proof: ODE2 more}. Assumption $\textbf{B2}$ ensures that the ODE system \eqref{eq:ODE2} is $\boldsymbol{\eta}_i$-identifiable for all $i=1, \ldots, p$. Consequently, $(\boldsymbol{x}_0, A, B\boldsymbol{z}_0^{*i}, BG\boldsymbol{z}_0^{*i}, \ldots, BG^{p-1}\boldsymbol{z}_0^{*i} )$ for all $i=1, \ldots, p$ is identifiable. Then, under assumption \textbf{B3}, the identifiability of matrix $B$ is established. To identify matrix $G$, assumption \textbf{B4} is required. While the ability to control the initial condition of the latent variables may appear strict, it is a reasonable assumption in our context. This is because identifying matrices $B$ and $G$ is necessary only when practitioners can intervene in the latent variables, thereby allowing control over their initial conditions. An alternative approach to identifying $B$ and $G$ involves intervening in the initial condition of each latent variable $z_i$ independently, rather than controlling the initial condition of all latent variables $\boldsymbol{z}$ simultaneously. This method draws inspiration from the "genetic single-node intervention" proposed by \cite{squires2023linear}, where one can intervene at each latent node individually. Further details of this method can be found in Appendix \ref{app:alternative approach}.

\subsection{Identifiability condition from discrete observations sampled from \textit{p} controllable trajectories}
We also extend the identifiability analysis of the ODE system \eqref{eq:ODE2} to cases where only discrete observations from $p$ controllable trajectories are available.
\begin{definition}\label{def: identifiability ODE2 discrete more}
Given $\boldsymbol{z}_0^{*i} \in \mathbb{R}^p$ for $i=1, \ldots, p$, for $\boldsymbol{x}_0 \in \mathbb{R}^d, A\in \repair{\mathcal A_\Delta}, B\in \mathbb{R}^{d\times p}$ and $G \in \mathbb{R}^{p\times p}$. For any $n \geqslant 1$, let $t_j, j = 1, \ldots, n$ be any $n$ time points and $\boldsymbol{x}_{ij} : = \boldsymbol{x}(t_j; \boldsymbol{\eta}_i)$ be the error-free observation of the trajectory $\boldsymbol{x}(\cdot; \boldsymbol{\eta}_i)$ at time $t_j$. Under the latent DAG assumption, we say the ODE system \eqref{eq:ODE2} is $\{\boldsymbol{\eta}_i\}_1^p$-identifiable from $\boldsymbol{x}_{i1}, \ldots, \boldsymbol{x}_{in}$, $i = 1, \ldots, p$, if for all $\boldsymbol{x}'_0 \in \mathbb{R}^d$, all $A'\in \repair{\mathcal A_\Delta}$, all $B' \in \mathbb{R}^{d\times p}$, and all $G' \in \mathbb{R}^{p\times p}$ with $(\boldsymbol{x}_0, A, B, G) \neq (\boldsymbol{x}'_0, A', B', G')$, it holds that $\exists i \in \{1, \ldots, p\}$ and $j \in \{1, \ldots, n\}$ such that $\boldsymbol{x}(t_j; \boldsymbol{\eta}_i){\neq} \boldsymbol{x}(t_j; \boldsymbol{\eta}'_i)$.
\end{definition}
Based on Definition \ref{def: identifiability ODE2 discrete more} we present the identifiability condition.
\begin{theorem}\label{thm: identifiability ODE2 discrete more}
Given $\boldsymbol{z}_0^{*i} \in \mathbb{R}^p$ for $i=1, \ldots, p$, for $\boldsymbol{x}_0 \in \mathbb{R}^d,  A\in \repair{\mathcal A_\Delta}, B\in \mathbb{R}^{d\times p}$ and $G \in \mathbb{R}^{p\times p}$. We define new observation $\boldsymbol{y}_{ij} := [\boldsymbol{x}_{ij}^T, 1, t_j, t_j^2,\ldots, t_j^{p-1}]^T\in \mathbb{R}^{d+p}$, for $i = 1, \ldots, p$ and $j = 1, \ldots, n$. Under the latent DAG assumption \repair{and \textbf{C0}}, the ODE system \eqref{eq:ODE2} is $\{\boldsymbol{\eta}_i\}_1^p$-identifiable from discrete observations $\boldsymbol{x}_{i1}, \ldots, \boldsymbol{x}_{in}$, $i=1, \ldots, p$, if assumptions \textbf{C2}, \textbf{B3} and \textbf{B4} are all satisfied.
\begin{itemize}
    \item [\textbf{C2}:] for each $i \in \{1, \ldots, p\}$ there exists $(d+p)$ $\boldsymbol{y}_{ij}$'s with indexes denoting as $\{j_{i1}, j_{i2}, \ldots, j_{i,d+p}\} \subseteq \{1, 2, \ldots, n\}$, such that the set of vectors $\{\boldsymbol{y}_{ij_{i1}}, \boldsymbol{y}_{ij_{i2}}, \ldots, \boldsymbol{y}_{ij_{i,d+p}}\}$ is linearly independent.
\end{itemize}
\end{theorem}

The proof of Theorem \ref{thm: identifiability ODE2 discrete more} can be found in Appendix \ref{proof: ODE2 discrete more}. Assumption \textbf{C2} ensures that the ODE system \eqref{eq:ODE2} is $\boldsymbol{\eta}_i$-identifiable from discrete observations $\boldsymbol{x}_{i1}, \ldots, \boldsymbol{x}_{in}$ for all $i=1, \ldots, p$. As in Subsection \ref{subsec:iden p continous}, under assumptions \textbf{B3} and \textbf{B4}, the matrices $B$ and $G$ are also identifiable.
\section{Simulations}\label{sec:simulations}
To evaluate the validity of the identifiability conditions established in Section 3 and 4, we present the results of simulations. As previously indicated, the ODE system \eqref{eq:ODE2} can be treated as a more intricate version of the ODE system \eqref{eq:ODE1}; hence, our simulation experiments are centered on the former. 

\textbf{Simulation design.} We conduct four sets of simulations, which include one identifiable case and one unidentifiable case for both the $\boldsymbol{\eta}$-identifiable check and the $\{\boldsymbol{\eta}_i\}_1^p$-identifiable check. The dimensions of both observable variables, $d$, and latent variables, $p$, are set to 3. The true underlying parameters of the systems are provided below. Observations are simulated from the true ODE systems for each case, with $n$ equally-spaced observations generated from the time interval $[0,1]$ for each trajectory, and we only keep the values of the observable variables $\boldsymbol{x}$.
\begin{equation*}
\begin{split}
      &A = \begin{bmatrix}
        2 & -2 & 1\\
        1 & 1 & -1\\
        1 & 0 & 2
    \end{bmatrix}\,, \ \ 
    B = \begin{bmatrix}
        -2 & -2 & 2\\
        0 & -1 & -2\\
        -1 & -1 & -2
    \end{bmatrix}\,, \ \ 
    G = \begin{bmatrix}
        0 & 2 & 1\\
        0 & 0 & -2\\
        0 & 0 & 0
    \end{bmatrix}\,, \ \
    A' = \begin{bmatrix}
        1 & 0 & 0\\
        0 & 1 & 0\\
        0 & 0 & 1
    \end{bmatrix}\,,  \\
    & \boldsymbol{x}_0 = \begin{bmatrix}
        -1\\
        1\\
        1
    \end{bmatrix}\,, \ \ \
    \boldsymbol{z}_0 = \begin{bmatrix}
        1\\
        -2\\
        -1
    \end{bmatrix}\,, \ \ \ 
    \boldsymbol{z}_0^{*1} = \begin{bmatrix}
        1\\
        0\\
        0
    \end{bmatrix}\,, \ \ \ 
    \boldsymbol{z}_0^{*2} = \begin{bmatrix}
        0\\
        1\\
        0
    \end{bmatrix}\,, \ \ \
    \boldsymbol{z}_0^{*3} = \begin{bmatrix}
        0\\
        0\\
        1
    \end{bmatrix}\,.\\
    & \boldsymbol{\eta}\mbox{-identifiable: } \boldsymbol{\eta} = (\boldsymbol{x}_0, \boldsymbol{z}_0, A, B, G), \mbox{ unidentifiable: } \boldsymbol{\eta} = (\boldsymbol{x}_0, \boldsymbol{z}_0, A', B, G)\,.\\
    & \{\boldsymbol{\eta}_i\}_1^p\mbox{-identifiable: } \boldsymbol{\eta}_i = (\boldsymbol{x}_0, \boldsymbol{z}_0^{*i}, A, B, G), \mbox{ unidentifiable: } \boldsymbol{\eta}_i = (\boldsymbol{x}_0, \boldsymbol{z}_0^{*i}, A', B, G), i =1, 2,  3\,.\\
\end{split}
\end{equation*}
\textbf{Parameter estimation.} The Nonlinear Least Squares (NLS) method is employed for parameter estimation, a widely used technique for estimating parameters in nonlinear regression models, including ODEs \cite{bock1983recent, jennrich1969asymptotic, xue2010sieve}. The \textit{"least\_squares"} function from the \textit{"scipy.optimize"} Python module, with default hyperparameter settings, is utilized for implementation. Given that the NLS loss function for our simulation is non-convex, parameter initialization is performed near the true values to promote convergence to the global minimum. Specifically, for the $\boldsymbol{\eta}$-(un)identifiable cases, initial parameter values are set to the true parameters plus a random value drawn from a uniform distribution $U(-0.1, 0.1)$ for each replication. For $\{\boldsymbol{\eta}_i\}_1^p$-(un)identifiable cases, initial parameter values are set to the true values plus a random value from $U(-0.3, 0.3)$.

\textbf{Evaluation metric.} Mean Squared Error (MSE) is adopted as the metric to assess the accuracy of the parameter estimator. To ensure the reliability of the estimation results, 100 independent random replications are run for each configuration, and we report the mean and variance of the squared error.

\textbf{Results analysis.} Table \ref{table1} and Table \ref{table2} present the simulation results for the $\boldsymbol{\eta}$-(un)identifiable cases and the  $\{\boldsymbol{\eta}_i\}_1^p$-(un)identifiable cases, respectively. According to Definition \ref{def: identifiability ODE2 continuous} and Definition \ref{def: identifiability ODE2 more}, for the $\boldsymbol{\eta}$-(un)identifiable cases, the identifiability of $(\boldsymbol{x}_0, A, B\boldsymbol{z}_0, BG\boldsymbol{z}_0, BG^2\boldsymbol{z}_0)$ needs to be checked, while for the  $\{\boldsymbol{\eta}_i\}_1^p$-(un)identifiable cases, we need to check the identifiability of $(\boldsymbol{x}_0, A, B, G)$. Since $\boldsymbol{x}_0$ is consistently identifiable (with MSE less than  $1.00$E-$10$) across all (un)identifiable cases, its results are not presented. 

In both Tables, for identifiable cases, as the number of samples $n$ increases, the MSEs for all parameters of interest decrease and approach zero. However, in the unidenfiable cases, where the identifiability condition \textbf{B1}/\textbf{B2} stated in Theorem \ref{thm: identifiability ODE2 continuous}/\ref{thm: identifiability ODE2 more} is unmet, the MSEs for certain parameters remain high irrespective of sample size. These results offer strong empirical support for the validity of the identifiability conditions outlined in Theorem \ref{thm: identifiability ODE2 continuous} and Theorem \ref{thm: identifiability ODE2 more}. It is noteworthy that in the  $\{\boldsymbol{\eta}_i\}_1^p$ case, where observations are sampled from $p=3$ controllable trajectories, remarkably accurate parameter estimates can be obtained even with a limited number of samples.

\begin{table}[htbp]
    \caption{MSEs of the $\boldsymbol{\eta}$-(un)identifiable cases of the ODE \eqref{eq:ODE2}}
    \label{table1}
    \centering
    \begin{tabu}{ccccccccc}
    \toprule
    \multicolumn{1}{c}{\multirow{2}{*}{$\boldsymbol{n}$}} & \multicolumn{4}{c}{\textbf{Identifiable}} & \multicolumn{4}{c}{\textbf{Unidentifiable}}\\
    \cmidrule(lr){2-5} \cmidrule(l){6-9}
       & $A$  & $B\boldsymbol{z}_0$  & $BG\boldsymbol{z}_0$  & $BG^2\boldsymbol{z}_0$ & $A$ &  $B\boldsymbol{z}_0$  & $BG\boldsymbol{z}_0$  & $BG^2\boldsymbol{z}_0$\\
    \midrule
    \addlinespace
    \multirow{2}{*}{10}   & 6.00\text{E-}05 & 0.0004  & 0.0044 & 0.0007 & 0.0994 & 0.0494 & 0.9185 & 0.6482 \\
    \rowfont{\scriptsize}   & ($\pm$5.40\text{E-}08) & ($\pm$3.45\text{E-}06) & ($\pm$0.0004) & ($\pm$3.91\text{E-}06)  & ($\pm$0.0157) & ($\pm$0.1243) & ($\pm$8.3148) & ($\pm$1.4306) \\
    \addlinespace
    \addlinespace
    \multirow{2}{*}{100}   & 4.15\text{E-}05 & 0.0003  & 0.0029 & 0.0005 & 0.0372 & 0.0174 & 0.3517 & 0.5767 \\
    \rowfont{\scriptsize}   & ($\pm$1.62\text{E-}08) & ($\pm$8.52\text{E-}07) & ($\pm$9.42\text{E-}05) & ($\pm$2.90\text{E-}06)  & ($\pm$0.0032) & ($\pm$0.0087) & ($\pm$0.3460) & ($\pm$1.4055) \\
    \addlinespace
    \addlinespace
    \multirow{2}{*}{500}   & 2.65\text{E-}05 & 0.0002  & 0.0019 & 0.0002 & 0.0461 & 0.1071 & 0.5783 & 0.3648 \\
    \rowfont{\scriptsize}   & ($\pm$8.71\text{E-}09) & ($\pm$4.38\text{E-}07) & ($\pm$4.84\text{E-}05) & ($\pm$8.38\text{E-}07)  & ($\pm$0.0099) & ($\pm$0.1768) & ($\pm$2.5747) & ($\pm$0.4507) \\
    \addlinespace
    \bottomrule
    \end{tabu}
\end{table}

\begin{table}[htbp]
    \caption{MSEs of the  $\{\boldsymbol{\eta}_i\}_1^p$-(un)identifiable cases of the ODE \eqref{eq:ODE2}}
    \label{table2}
    \centering
    \begin{tabu}{ccccccc}
    \toprule
    \multicolumn{1}{c}{\multirow{2}{*}{$\boldsymbol{n}$}} & \multicolumn{3}{c}{\textbf{Identifiable}} & \multicolumn{3}{c}{\textbf{Unidentifiable}}\\
    \cmidrule(lr){2-4} \cmidrule(l){5-7}
       & $A$  & $B$ & $G$ &  $A$ &  $B$  & $G$ \\
    \midrule
    \addlinespace
    \multirow{2}{*}{10}   & 5.83\text{E-}22 & 2.85\text{E-}21  & 2.27\text{E-}21 & 0.6349 & 0.1913 & 0.0044  \\
    \rowfont{\footnotesize}   & ($\pm$7.41\text{E-}42) & ($\pm$2.75\text{E-}40) & ($\pm$5.69\text{E-}41) & ($\pm$0.7464)  & ($\pm$0.0686) & ($\pm$0.0011)  \\
    \addlinespace
    \addlinespace
    \multirow{2}{*}{30}   & 1.50\text{E-}22 & 7.80\text{E-}22  & 5.76\text{E-}22 & 0.6169 & 0.1850 & 0.0045 \\
    \rowfont{\footnotesize}   & ($\pm$3.23\text{E-}43) & ($\pm$1.14\text{E-}41) & ($\pm$5.28\text{E-}42) & ($\pm$0.7194)  & ($\pm$0.0657) & ($\pm$0.0007)  \\
    \addlinespace
    \addlinespace
    \multirow{2}{*}{50}   & 5.16\text{E-}23 & 3.01\text{E-}22 & 2.39\text{E-}22 & 0.5876 & 0.1761 & 0.0045  \\
    \rowfont{\footnotesize}   & ($\pm$6.20\text{E-}44) & ($\pm$3.27\text{E-}42) & ($\pm$8.46\text{E-}43) & ($\pm$0.6895)  & ($\pm$0.0627) & ($\pm$0.0008)  \\
    \addlinespace
    \bottomrule
    \end{tabu}
\end{table}

For the $\boldsymbol{\eta}$-(un)identifiable cases, assumption \textbf{C1} stated in Theorem \ref{thm: identifiability ODE2 discrete} holds true for all values of $n$ in the identifiable cases, while it is violated across all $n$ in the unidentifiable cases. In the  $\{\boldsymbol{\eta}_i\}_1^p$-(un)identifiable cases, condition \textbf{C2} stated in Theorem \ref{thm: identifiability ODE2 discrete more} is satisfied for all values of $n$ in the identifiable cases, but is found to be violated for all values of $n$ in the unidentifiable cases. These findings provide strong empirical evidence supporting the validity of the identifiability conditions proposed in Theorem \ref{thm: identifiability ODE2 discrete} and Theorem \ref{thm: identifiability ODE2 discrete more}. 

In Appendix \ref{app:more simulations}, we present additional simulation results for higher-dimensional cases, along with simulations that incorporate a variety of ground-truth parameter configurations. These results consistently affirm the validity of our proposed identifiability conditions. For further details, please refer to Appendix \ref{app:more simulations}.

\section{Related work}

\textbf{Identifiability analysis of linear ODE systems.}
Within control theory, extensive research has been conducted on the identifiability analysis of linear dynamical systems governed by ODEs \cite{bellman1970structural, gargash1980necessary, glover1974parametrizations, grewal1976identifiability}. In the applied mathematics area, Stanhope et al. \cite{stanhope2014identifiability} and Qiu et al. \cite{qiu2022identifiability} have systematically investigated the identifiability of linear ODE systems based on a single trajectory. Furthermore, Wang et al. \cite{wang2024identifiability} have extended these findings to scenarios where only discrete observations sampled from a single trajectory are available. However, existing studies primarily concentrate on linear ODE systems with fully observable variables. To the best of our knowledge, our work represents the inaugural endeavor to systematically analyze the identifiability of linear ODE systems in the presence of hidden confounders.

\textbf{Connection between causality and differential equations.} Differential equations provide a natural framework for understanding causality within dynamic systems, particularly in the context of continuous-time processes \cite{aalen2012causality, scholkopf2021toward}. Consequently, significant efforts have been directed towards establishing a theoretical link between causality and differential equations. In the deterministic case, Mooij et al. \cite{mooij2013ordinary} and Rubenstein et al. \cite{rubenstein2018deterministic} have established a mathematical connection between ODEs and Structural Causal Models (SCMs). Wang et al. \cite{wang2024identifiability} have proposed a method for inferring the causal structure of linear ODEs. In the domain of neural ODEs, Aliee et al. \cite{aliee2022sparsity, aliee2021beyond} have applied various regularization techniques to enhance the recovery of the causal relationships. Turning to the stochastic case, 
Hansen et al. \cite{hansen2014causal} and Wang et al. \cite{wang2024generator} have proposed causal interpretations and identifiability analysis of Stochastic Differential Equations (SDEs). Additionally, Bellot et al. \cite{bellot2021neural} have introduced a method for consistently discovering the causal structure of SDE systems using penalized neural ODEs. These works aim to establish a theoretical connection between causality and differential equations in various ways. Our contribution to this scholarly landscape lies in the systematic analysis of the identifiability of linear ODEs, particularly in the presence of hidden confounders.

\section{Conclusion}\label{sec:conclusion}

This paper presents a systematic identifiability analysis of linear ODE systems incorporating hidden confounders. Specifically, we establish a sufficient and necessary identifiability condition for the linear ODE system with independent latent confounders. Additionally, we provide four identifiability conditions for the linear ODE system with causally related latent confounders, wherein the causal structure of the latent confounders adheres to a DAG.

A notable limitation of our work lies in the practical verification of these identifiability conditions, given that the true underlying system parameters are often unavailable in real-world scenarios. However, our study significantly contributes to the understanding of the intrinsic structure of linear ODE systems with hidden confounders. By providing insights into the identifiability aspects, our findings empower practitioners to utilize models that adhere to the proposed conditions (e.g., through constrained parameter estimation) for learning from real-world data while ensuring identifiability. 
\section*{Acknowledgements}
\repair{We thank Leo Gao for pointing out the error in our original identifiability results for discrete observations.}

YW was supported by the Australian Government Research Training Program (RTP) Scholarship from the University of Melbourne. BH was supported by NSF DMS-2428058. XG was supported by ARC DE210101352. MG was supported by ARC DE210101624 and ARC DP240102088.

\newpage
\bibliographystyle{abbrv}
\bibliography{references}

@article{stanhope2014identifiability,
  title={Identifiability of linear and linear-in-parameters dynamical systems from a single trajectory},
  author={Stanhope, Shelby and Rubin, Jonathan E and Swigon, David},
  journal={SIAM Journal on Applied Dynamical Systems},
  volume={13},
  number={4},
  pages={1792--1815},
  year={2014},
  publisher={SIAM}
}

@article{qiu2022identifiability,
  title={Identifiability analysis of linear ordinary differential equation systems with a single trajectory},
  author={Qiu, Xing and Xu, Tao and Soltanalizadeh, Babak and Wu, Hulin},
  journal={Applied Mathematics and Computation},
  volume={430},
  pages={127260},
  year={2022},
  publisher={Elsevier}
}

@article{wang2024identifiability,
  title={Identifiability and Asymptotics in Learning Homogeneous Linear ODE Systems from Discrete Observations},
  author={Wang, Yuanyuan and Huang, Wei and Gong, Mingming and Geng, Xi and Liu, Tongliang and Zhang, Kun and Tao, Dacheng},
  journal={Journal of Machine Learning Research},
  volume={25},
  number={154},
  pages={1--50},
  year={2024}
}

@article{scholkopf2021toward,
  title={Toward causal representation learning},
  author={Sch{\"o}lkopf, Bernhard and Locatello, Francesco and Bauer, Stefan and Ke, Nan Rosemary and Kalchbrenner, Nal and Goyal, Anirudh and Bengio, Yoshua},
  journal={Proceedings of the IEEE},
  volume={109},
  number={5},
  pages={612--634},
  year={2021},
  publisher={IEEE}
}

@article{koleva2010two,
  title={Two-grid quasilinearization approach to ODEs with applications to model problems in physics and mechanics},
  author={Koleva, Miglena N and Vulkov, Lubin G},
  journal={Computer Physics Communications},
  volume={181},
  number={3},
  pages={663--670},
  year={2010},
  publisher={Elsevier}
}

@article{zhong2019symplectic,
  title={Symplectic ode-net: Learning hamiltonian dynamics with control},
  author={Zhong, Yaofeng Desmond and Dey, Biswadip and Chakraborty, Amit},
  journal={arXiv preprint arXiv:1909.12077},
  year={2019}
}

@article{mandelzweig2001quasilinearization,
  title={Quasilinearization approach to nonlinear problems in physics with application to nonlinear ODEs},
  author={Mandelzweig, VB and Tabakin, F},
  journal={Computer Physics Communications},
  volume={141},
  number={2},
  pages={268--281},
  year={2001},
  publisher={Elsevier}
}

@inproceedings{bzhikhatlov2017research,
  title={Research of robot model behaviour depending on model parameters using physic engines bullet physics and ode},
  author={Bzhikhatlov, I and Perepelkina, S},
  booktitle={2017 International Conference on Industrial Engineering, Applications and Manufacturing (ICIEAM)},
  pages={1--4},
  year={2017},
  organization={IEEE}
}

@article{gratie2013ode,
  title={ODE analysis of biological systems},
  author={Gratie, Diana-Elena and Iancu, Bogdan and Petre, Ion},
  journal={Formal Methods for Dynamical Systems: 13th International School on Formal Methods for the Design of Computer, Communication, and Software Systems, SFM 2013, Bertinoro, Italy, June 17-22, 2013. Advanced Lectures},
  pages={29--62},
  year={2013},
  publisher={Springer}
}

@article{polynikis2009comparing,
  title={Comparing different ODE modelling approaches for gene regulatory networks},
  author={Polynikis, Athanasios and Hogan, SJ and Di Bernardo, Mario},
  journal={Journal of theoretical biology},
  volume={261},
  number={4},
  pages={511--530},
  year={2009},
  publisher={Elsevier}
}

@article{quach2007estimating,
  title={Estimating parameters and hidden variables in non-linear state-space models based on ODEs for biological networks inference},
  author={Quach, Minh and Brunel, Nicolas and d'Alch{\'e}-Buc, Florence},
  journal={Bioinformatics},
  volume={23},
  number={23},
  pages={3209--3216},
  year={2007},
  publisher={Oxford University Press}
}

@article{stadter2021benchmarking,
  title={Benchmarking of numerical integration methods for ODE models of biological systems},
  author={St{\"a}dter, Philipp and Sch{\"a}lte, Yannik and Schmiester, Leonard and Hasenauer, Jan and Stapor, Paul L},
  journal={Scientific reports},
  volume={11},
  number={1},
  pages={2696},
  year={2021},
  publisher={Nature Publishing Group UK London}
}

@article{su2021deep,
  title={Deep learning of biological models from data: applications to ODE models},
  author={Su, Wei-Hung and Chou, Ching-Shan and Xiu, Dongbin},
  journal={Bulletin of Mathematical Biology},
  volume={83},
  pages={1--19},
  year={2021},
  publisher={Springer}
}

@article{tsoularis2021some,
  title={On some important ordinary differential equations of dynamic economics},
  author={Tsoularis, Anastasios},
  journal={Recent developments in the solution of nonlinear differential equations},
  pages={147--153},
  year={2021}
}

@book{dockner2000differential,
  title={Differential games in economics and management science},
  author={Dockner, Engelbert},
  year={2000},
  publisher={Cambridge University Press}
}

@book{weber2011optimal,
  title={Optimal control theory with applications in economics},
  author={Weber, Thomas A},
  year={2011},
  publisher={MIT press}
}

@book{tu2012dynamical,
  title={Dynamical systems: an introduction with applications in economics and biology},
  author={Tu, Pierre NV},
  year={2012},
  publisher={Springer Science \& Business Media}
}

@article{jennrich1969asymptotic,
  title={Asymptotic properties of non-linear least squares estimators},
  author={Jennrich, Robert I},
  journal={The Annals of Mathematical Statistics},
  volume={40},
  number={2},
  pages={633--643},
  year={1969},
  publisher={JSTOR}
}

@article{xue2010sieve,
  title={Sieve estimation of constant and time-varying coefficients in nonlinear ordinary differential equation models by considering both numerical error and measurement error},
  author={Xue, Hongqi and Miao, Hongyu and Wu, Hulin},
  journal={Annals of statistics},
  volume={38},
  number={4},
  pages={2351},
  year={2010},
  publisher={NIH Public Access}
}

@article{hoyer2008estimation,
  title={Estimation of causal effects using linear non-Gaussian causal models with hidden variables},
  author={Hoyer, Patrik O and Shimizu, Shohei and Kerminen, Antti J and Palviainen, Markus},
  journal={International Journal of Approximate Reasoning},
  volume={49},
  number={2},
  pages={362--378},
  year={2008},
  publisher={Elsevier}
}

@article{chen2021fritl,
  title={Fritl: A hybrid method for causal discovery in the presence of latent confounders},
  author={Chen, Wei and Zhang, Kun and Cai, Ruichu and Huang, Biwei and Ramsey, Joseph and Hao, Zhifeng and Glymour, Clark},
  journal={arXiv preprint arXiv:2103.14238},
  year={2021}
}

@article{straubing1983combinatorial,
  title={A combinatorial proof of the Cayley-Hamilton theorem},
  author={Straubing, Howard},
  journal={Discrete Mathematics},
  volume={43},
  number={2-3},
  pages={273--279},
  year={1983},
  publisher={Elsevier}
}

@book{teschl2012ordinary,
  title={Ordinary differential equations and dynamical systems},
  author={Teschl, Gerald},
  volume={140},
  year={2012},
  publisher={American Mathematical Soc.}
}

@inproceedings{squires2023linear,
  title={Linear causal disentanglement via interventions},
  author={Squires, Chandler and Seigal, Anna and Bhate, Salil S and Uhler, Caroline},
  booktitle={International Conference on Machine Learning},
  pages={32540--32560},
  year={2023},
  organization={PMLR}
}

@article{bellman1970structural,
  title={On structural identifiability},
  author={Bellman, Ror and {\AA}str{\"o}m, Karl Johan},
  journal={Mathematical biosciences},
  volume={7},
  number={3-4},
  pages={329--339},
  year={1970},
  publisher={Elsevier}
}

@article{gargash1980necessary,
  title={A necessary and sufficient condition of global structural identifiability of compartmental models},
  author={Gargash, BC and Mital, DP},
  journal={Computers in biology and medicine},
  volume={10},
  number={4},
  pages={237--242},
  year={1980},
  publisher={Elsevier}
}

@article{glover1974parametrizations,
  title={Parametrizations of linear dynamical systems: Canonical forms and identifiability},
  author={Glover, Keith and Willems, Jan},
  journal={IEEE Transactions on Automatic Control},
  volume={19},
  number={6},
  pages={640--646},
  year={1974},
  publisher={IEEE}
}

@article{grewal1976identifiability,
  title={Identifiability of linear and nonlinear dynamical systems},
  author={Grewal, MS and Glover, Keith},
  journal={IEEE Transactions on automatic control},
  volume={21},
  number={6},
  pages={833--837},
  year={1976},
  publisher={IEEE}
}

@article{cai2019triad,
  title={Triad constraints for learning causal structure of latent variables},
  author={Cai, Ruichu and Xie, Feng and Glymour, Clark and Hao, Zhifeng and Zhang, Kun},
  journal={Advances in neural information processing systems},
  volume={32},
  year={2019}
}

@article{xie2020generalized,
  title={Generalized independent noise condition for estimating latent variable causal graphs},
  author={Xie, Feng and Cai, Ruichu and Huang, Biwei and Glymour, Clark and Hao, Zhifeng and Zhang, Kun},
  journal={Advances in neural information processing systems},
  volume={33},
  pages={14891--14902},
  year={2020}
}

@inproceedings{kaltenpoth2023causal,
  title={Causal discovery with hidden confounders using the algorithmic Markov condition},
  author={Kaltenpoth, David and Vreeken, Jilles},
  booktitle={Uncertainty in Artificial Intelligence},
  pages={1016--1026},
  year={2023},
  organization={PMLR}
}

@inproceedings{xie2022identification,
  title={Identification of linear non-gaussian latent hierarchical structure},
  author={Xie, Feng and Huang, Biwei and Chen, Zhengming and He, Yangbo and Geng, Zhi and Zhang, Kun},
  booktitle={International Conference on Machine Learning},
  pages={24370--24387},
  year={2022},
  organization={PMLR}
}

@article{huang2022latent,
  title={Latent hierarchical causal structure discovery with rank constraints},
  author={Huang, Biwei and Low, Charles Jia Han and Xie, Feng and Glymour, Clark and Zhang, Kun},
  journal={Advances in Neural Information Processing Systems},
  volume={35},
  pages={5549--5561},
  year={2022}
}

@incollection{verma1990causal,
  title={Causal networks: Semantics and expressiveness},
  author={Verma, Thomas and Pearl, Judea},
  booktitle={Machine intelligence and pattern recognition},
  volume={9},
  pages={69--76},
  year={1990},
  publisher={Elsevier}
}

@book{pearl2009causality,
  title={Causality},
  author={Pearl, Judea},
  year={2009},
  publisher={Cambridge university press}
}

@inproceedings{bock1983recent,
  title={Recent advances in parameteridentification techniques for ode},
  author={Bock, Hans Georg},
  booktitle={Numerical Treatment of Inverse Problems in Differential and Integral Equations: Proceedings of an International Workshop, Heidelberg, Fed. Rep. of Germany, August 30—September 3, 1982},
  pages={95--121},
  year={1983},
  organization={Springer}
}

@article{aalen2012causality,
  title={Causality, mediation and time: a dynamic viewpoint},
  author={Aalen, Odd O and R{\o}ysland, Kjetil and Gran, Jon Michael and Ledergerber, Bruno},
  journal={Journal of the Royal Statistical Society: Series A (Statistics in Society)},
  volume={175},
  number={4},
  pages={831--861},
  year={2012},
  publisher={Wiley Online Library}
}

@inproceedings{mooij2013ordinary,
  title={From Ordinary Differential Equations to Structural Causal Models: the deterministic case},
  author={Mooij, JM and Janzing, D and Sch{\"o}lkopf, B},
  booktitle={29th Conference on Uncertainty in Artificial Intelligence (UAI 2013)},
  pages={440--448},
  year={2013},
  organization={AUAI Press}
}

@inproceedings{rubenstein2018deterministic,
  title={From Deterministic ODEs to Dynamic Structural Causal Models},
  author={Rubenstein, P and Bongers, S and Sch{\"o}lkopf, B and Mooij, JM},
  booktitle={34th Conference on Uncertainty in Artificial Intelligence (UAI 2018)},
  pages={114--123},
  year={2018},
  organization={Curran Associates, Inc.}
}

@article{hansen2014causal,
  title={Causal interpretation of stochastic differential equations},
  author={Hansen, Niels and Sokol, Alexander},
  journal={Electronic Journal of Probability},
  volume={19},
  pages={1--24},
  year={2014},
  publisher={Institute of Mathematical Statistics and Bernoulli Society}
}

@inproceedings{bellot2021neural,
  title={Neural graphical modelling in continuous-time: consistency guarantees and algorithms},
  author={Bellot, Alexis and Branson, Kim and van der Schaar, Mihaela},
  booktitle={International Conference on Learning Representations},
  year={2021}
}

@article{wang2024generator,
  title={Generator Identification for Linear SDEs with Additive and Multiplicative Noise},
  author={Wang, Yuanyuan and Geng, Xi and Huang, Wei and Huang, Biwei and Gong, Mingming},
  journal={Advances in Neural Information Processing Systems},
  volume={36},
  year={2024}
}

@article{aliee2021beyond,
  title={Beyond predictions in neural odes: Identification and interventions},
  author={Aliee, Hananeh and Theis, Fabian J and Kilbertus, Niki},
  journal={arXiv preprint arXiv:2106.12430},
  year={2021}
}

@article{aliee2022sparsity,
  title={Sparsity in continuous-depth neural networks},
  author={Aliee, Hananeh and Richter, Till and Solonin, Mikhail and Ibarra, Ignacio and Theis, Fabian and Kilbertus, Niki},
  journal={Advances in Neural Information Processing Systems},
  volume={35},
  pages={901--914},
  year={2022}
}

@article{burnham2008inference,
  title={Inference of chemical reaction networks},
  author={Burnham, Samantha C and Searson, Dominic P and Willis, Mark J and Wright, Allen R},
  journal={Chemical Engineering Science},
  volume={63},
  number={4},
  pages={862--873},
  year={2008},
  publisher={Elsevier}
}

@article{craciun2005multiple,
  title={Multiple equilibria in complex chemical reaction networks: I. The injectivity property},
  author={Craciun, Gheorghe and Feinberg, Martin},
  journal={SIAM Journal on Applied Mathematics},
  volume={65},
  number={5},
  pages={1526--1546},
  year={2005},
  publisher={SIAM}
}

@article{anisiu2014lotka,
  title={Lotka, Volterra and their model},
  author={Anisiu, Mira-Cristiana},
  journal={Did{\'a}ctica mathematica},
  volume={32},
  number={01},
  year={2014}
}

\newpage
\appendix
\part{Appendix for "Identifiability Analysis of Linear ODE Systems with Hidden Confounders"}
\section{Summary of notations and proposed identifiability conditions}

\begin{table}[htpb]
    \caption{Summary of notations}
    \label{table-notations}
    \centering
    \begin{tabu}{ll}
    \toprule
     \textbf{Notation} & \textbf{Description} \\
    \midrule
     $\boldsymbol{x/z}$  &  observable/latent variables \\
     $x_i/z_i$  & the $i$-th observable/latent variable \\
     $t$  & time\\
     $t_j$ & the $j$-th time point \\
     $\boldsymbol{x}(t)/\boldsymbol{z}(t)$ & state of observable/latent variable at time $t$\\
     $\boldsymbol{x}_j$ & $\boldsymbol{x}(t_j)$, observable state at time $t_j$\\
     $\boldsymbol{x}_0/ \boldsymbol{z}_0$ & initial condition of observable/latent variable\\
     $\dot{\boldsymbol{x}}(t)$  & first derivative of $\boldsymbol{x}(t)$ w.r.t. time $t$\\
     $d$ & dimension of observable variables\\
     $p$  & dimension of latent variables\\
     $A,B,G$  & constant parameter matrices defined in Eq.\eqref{eq:ODE1} and \eqref{eq:ODE2} \\
     $\boldsymbol{f}(t)$ & Function of time $t$ defined in Eq.\eqref{eq:ODE1} \\
     $\boldsymbol{v}_k$  & constant parameter vector defined in Eq.\eqref{eq:f(t)} \\
     $\{\boldsymbol{v}_k\}_0^r$ & all the $\boldsymbol{v}_k$'s for $k=0,\ldots,r$ \\
     $\boldsymbol{\theta}$ & $:= (\boldsymbol{x}_0, \boldsymbol{z}_0, A, B, \{\boldsymbol{v}_k\}_0^r)$, the system parameter of ODE system \eqref{eq:ODE1} \\
     $\boldsymbol{\beta}$ & a vector defined in Thm.\ref{thm:identifiability ODE1} \textbf{A1}\\
     $\boldsymbol{y}(t)$ & augmented state \\
     $\boldsymbol{y}_0$ & initial condition of augmented variable\\
     $\boldsymbol{\eta}$ & $:= (\boldsymbol{x}_0, \boldsymbol{z}_0, A, B, G)$, the system parameter of ODE system \eqref{eq:ODE2}\\
    $\boldsymbol{\gamma}$ & a vector defined in Thm.\ref{thm: identifiability ODE2 continuous} \textbf{B1} \\
     $\boldsymbol{z}_0^{*}$ & given initial condition of latent variable\\
     $\boldsymbol{z}_0^{*i}$ & the $i$-th given initial condition of latent variable\\
     $\boldsymbol{\eta}_i$ & $:=(\boldsymbol{x}_0, \boldsymbol{z}_0^{*i}, A, B, G)$, the system parameter of ODE system \eqref{eq:ODE2}\\
     $\boldsymbol{\gamma}_i$ & a vector defined in Thm \ref{thm: identifiability ODE2 more} \textbf{B2} \\
     $\boldsymbol{x}_{ij}$ & $:= \boldsymbol{x}(t_j;\boldsymbol{\eta}_i)$, observable state of ODE system \eqref{eq:ODE2} with parameter $\boldsymbol{\eta}_i$ at time $t_j$\\
     $\boldsymbol{y}_{ij}$ & augmented state of $\boldsymbol{x}_{ij}$ at time $t_j$ \\
      $A', \boldsymbol{x}_0', \ldots$  & the alternative counterpart corresponding to $A,\boldsymbol{x}_0, \ldots$ \\ 
    \bottomrule
    \end{tabu}
\end{table}

\begin{table}[htpb]
    \caption{Summary of proposed identifiability conditions\repair{. The discrete-observation rows additionally require \textbf{C0} and identifiability within $\mathcal A_\Delta$.}}
    \label{table-thms}
    \centering
    \begin{tabu}{cccccc}
    \toprule
     \textbf{ODEs} & \textbf{Conds.}  & \textbf{\# Traj.} & \textbf{Obs.} &  
     \textbf{Def./Thm.} & \textbf{Necessity}  \\
    \midrule
    Eq.\eqref{eq:ODE1}+\eqref{eq:f(t)} & A1 & 1 & continuous & 3.1 & Yes\\
    Eq.\eqref{eq:ODE2} & latent DAG, B1 & 1 & continuous & 4.1 & Yes\\
    Eq.\eqref{eq:ODE2} & latent DAG, C1 & 1 & discrete & 4.2 & Yes\\
    Eq.\eqref{eq:ODE2} & latent DAG, B2, B3, B4 & $p$ & continuous & 4.3 & No\\
    Eq.\eqref{eq:ODE2} & latent DAG, C2, B3, B4 & $p$ & discrete & 4.4 & No\\
    \bottomrule
    \end{tabu}
\end{table}
\newpage
\section{Real world examples}\label{app:real world examples}

In this section,  we present two real-world examples that correspond to the ODE systems \eqref{eq:ODE1} and \eqref{eq:ODE2}. These examples initially assume fully observable systems, with latent variables introduced by us based on prior experience or established physical laws.

\subsection{Damped harmonic oscillators model}
Consider a one-dimensional system comprising $D$ point masses $m_i$ for $i = 1, \ldots, D$ with positions $Q_i(t) \in \mathbb{R}$ and momenta $P_i(t) \in \mathbb{R}$. These masses are interconnected by springs characterized by spring constants $k_i$ and equilibrium lengths $l_i$, and each mass is subject to friction with coefficient $b_i$. The system's boundary conditions are fixed at $Q_0(t) = 0$ and $Q_{D+1}(t) = L$.

The dynamics of this system are described by the following linear ODE system \cite{mooij2013ordinary}:

\begin{equation}
\begin{split}
    \dot{P}_i(t) &= k_i(Q_{i+1}(t) - Q_i(t) -l_i)-k_{i-1}(Q_i(t) -Q_{i-1}(t)-l_{i-1}) - b_i P_i(t)/m_i \\
    \dot{Q}_i(t) &= P_i(t)/m_i
\end{split}    
\end{equation}
where $Q_0(t) = 0$ and $Q_{D+1}(t) = L$ represent the fixed boundary conditions. External forces $F_j(t)$ (e.g., wind force or a varying magnetic field) may influence the entire system of coupled oscillators. These external forces can be modeled here as latent variables with constant derivatives. Consequently, the system can be reformulated as follows:
\begin{equation}
\begin{split}
    &\dot{P}_i(t) = k_i(Q_{i+1}(t) - Q_i(t) -l_i)-k_{i-1}(Q_i(t) -Q_{i-1}(t)-l_{i-1}) - b_i P_i(t)/m_i + \sum_{j}\alpha_{ij} F_j(t) \\
    &\dot{Q}_i(t) = P_i(t)/m_i\\
    &\dot{F}_j(t) = c_j
\end{split}    
\end{equation}

where $\alpha_{ij}$ is a constant determining the effect of the external force $F_j(t)$ on the $i$-th mass, and $c_j$ is the constant rate of change of the external force $F_j(t)$. This model aligns with our ODE system (2), and an illustrative causal structure for this model is provided in Figure \ref{fig:oscillator}.

\begin{figure}[ht]
  \centering
  \includegraphics[width=0.4\textwidth]{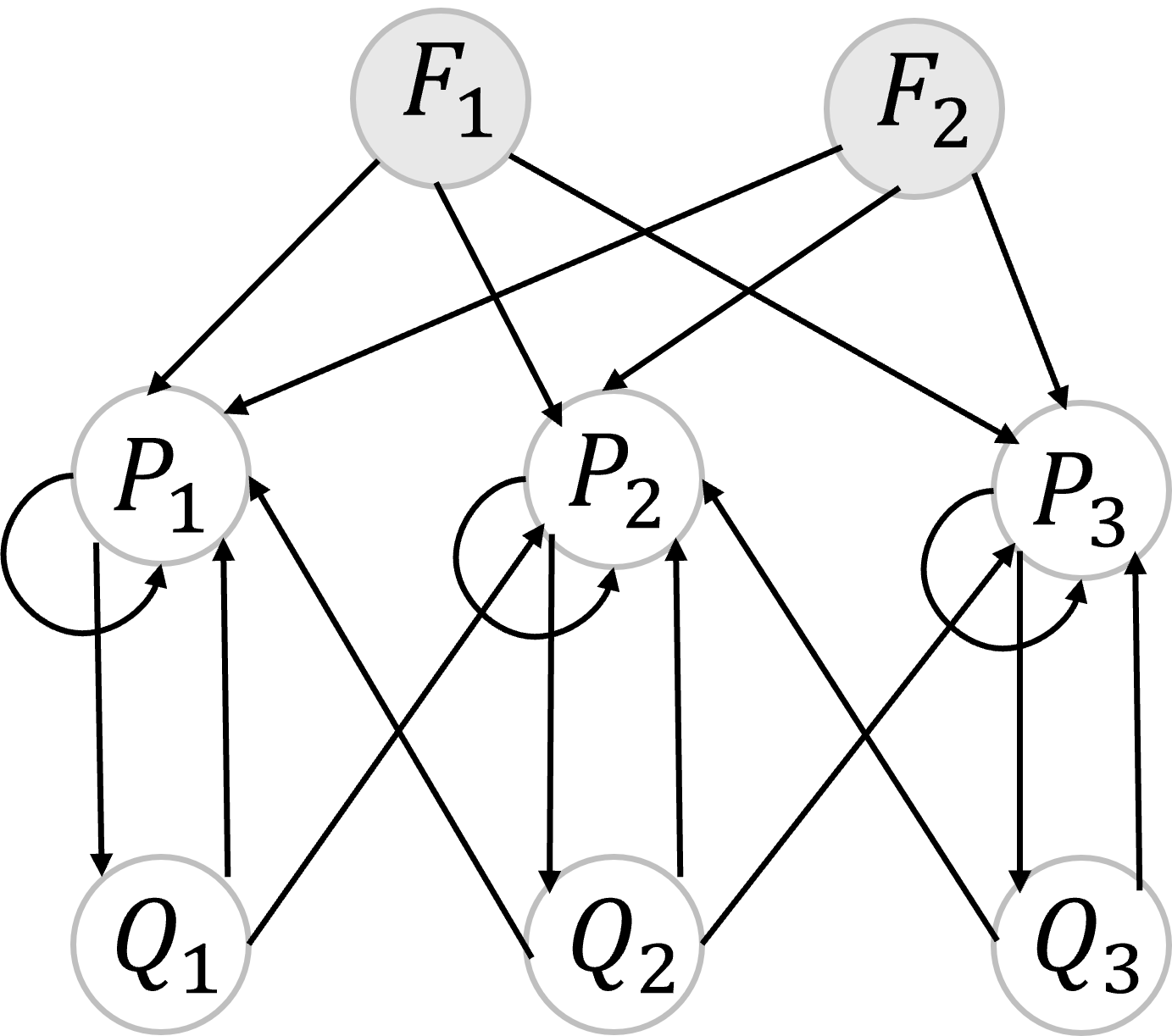}
  \caption{Example causal structure of the damped harmonic oscillators model with $3$ oscillators and $2$ latent variables.}
  \label{fig:oscillator}
\end{figure}

In regions with predictable wind patterns, such as during monsoon seasons or in controlled experimental settings, wind force can be approximated with a constant rate, making this an ideal context for modeling external forces with constant derivatives. Furthermore, constant forces or those represented as polynomial functions of time align well with our ODE system structure. For instance, a uniform magnetic field acting on the system would produce a constant force. These examples demonstrate that various latent factors can effectively fit within our ODE structure.

\subsection{Population model}

The growth of a population $P(t)$ can be described by a linear ODE \cite{anisiu2014lotka}:
\begin{equation*}
    \dot{P}(t) = a P(t),
\end{equation*}
where $a$ is a constant representing the population growth rate. This system may also be influenced by latent variables $L_i$, such as environmental factors or food supply. By incorporating these latent influences, the system can be expressed as:
\begin{equation*}
\begin{split}
    \dot{P}(t) &= a P(t) + b L_1(t) + c L_2 (t)\\
    \dot{L_1}(t) &= l L_2(t)\\
    \dot{L_2}(t) &= m  
\end{split}   
\end{equation*}
where $a, b, c, l$ and $m$ are constants. Here, $L_1(t)$ represents the food supply, which is influenced by the environmental factor $L_2(t)$. $L_2(t)$ corresponds to an environmental factor, such as temperature or pollution level, that changes steadily over time. This model aligns well with our ODE system (3),  and an illustrative causal structure for this model is provided in Figure \ref{fig:population}.

\begin{figure}[ht]
  \centering
  \includegraphics[width=0.2\textwidth]{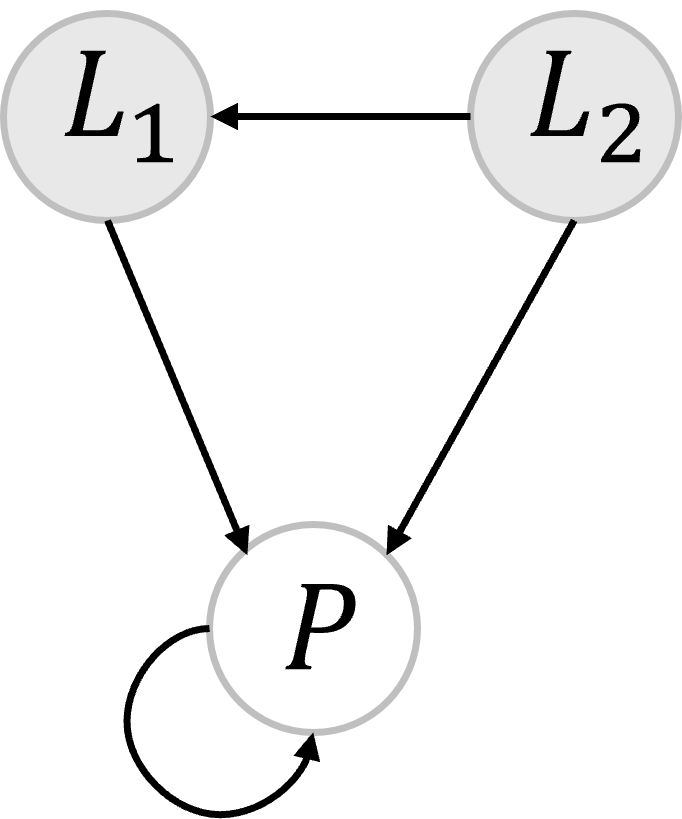}
  \caption{Causal structure of the population model.}
  \label{fig:population}
\end{figure}

An example of an environmental factor changing at a constant rate is pollution from an industrial plant that continuously releases a fixed amount of pollutants, or from a wastewater treatment plant that discharges a specified amount of treated wastewater into a river on an hourly basis.

\newpage
\section{An example of an unidentifiable case of the linear ODE system (\ref{eq:ODE2})}\label{app:unidentifiable example}

Recall that the parameters of the ODE system \eqref{eq:ODE2} are:
\begin{equation*}
    \boldsymbol{x}_0 = \begin{bmatrix}
     1\\
     1
    \end{bmatrix}\,,\ \ 
    \boldsymbol{z}_0 = \begin{bmatrix}
     1\\
     1
    \end{bmatrix}\,, \ \ 
        B = \begin{bmatrix}
        1 & 1\\
        1 & 1
    \end{bmatrix}\,,\ \ 
    G = \begin{bmatrix}
        0 & 1\\
        0 & 0
    \end{bmatrix}\,,
\end{equation*}
\begin{equation*}
        A = \begin{bmatrix}
        1 & 0\\
        0 & 1
    \end{bmatrix}\,,\ \ \ 
    A' = \begin{bmatrix}
    0 & 1\\
    1 & 0
\end{bmatrix}\,, \ \ \ 
    M = \begin{bmatrix}
    A & B\\
    \boldsymbol{0} & G
    \end{bmatrix}\,, \ \ \ 
    M' = \begin{bmatrix}
    A' & B\\
    \boldsymbol{0} & G
    \end{bmatrix}\,.
\end{equation*}

We first calculate the solution of $\boldsymbol{z}(t)$,
\begin{equation*}
\begin{split}
    \boldsymbol{z}(t) &= e^{Gt}\boldsymbol{z}_0   \\
    &= \sum_{k=0}^\infty \cfrac{G^k \boldsymbol{z}_0}{k!}t^k
    = \sum_{k=0}^1 \cfrac{G^k \boldsymbol{z}_0}{k!}t^k
    = \begin{bmatrix}
        1+t \\
        1
    \end{bmatrix}
\end{split}
\end{equation*}

We intervene $x_1(t) = 1$, then under matrix $M$:
\begin{equation*}
\begin{split}
    \dot{x}_2(t) &= x_2(t) + z_1(t) + z_2(t)\\
    &= x_2(t) + t + 2\,.
\end{split}
\end{equation*}

To solve this differential equation, we rewrite it in the standard linear form and multiply both sides by the integrating factor $e^{-t}$,
\begin{equation*}
    e^{-t}\dot{x}_2(t) - e^{-t}x_2(t) = (t + 2)e^{-t}\,.
\end{equation*}
The left-hand side of this equation is the derivative of $e^{-t}x_2(t)$:
\begin{equation*}
    \cfrac{d}{dt}(e^{-t}x_2(t)) = (t + 2)e^{-t}\,.
\end{equation*}
Next, integrate both sides w.r.t. $t$:
\begin{equation*}
 \int  \cfrac{d}{dt}(e^{-t}x_2(t)) dt = \int (t + 2)e^{-t} dt\,.
\end{equation*}
The left-hand side integrates to:
\begin{equation*}
    e^{-t}x_2(t)\,.
\end{equation*}
Next, we use integration by parts to find the integral on the right-hand side:
\begin{equation*}
\begin{split}
    \int (t + 2)e^{-t} dt &= -(t+2)e^{-t} - \int -e^{-t}dt    \\
    &= -(t+2)e^{-t} - e^{-t}\\
    &=  -(t+3)e^{-t}\,.
\end{split}
\end{equation*}
Thus:
\begin{equation*}
     e^{-t}x_2(t) =  -(t+3)e^{-t} + C\,,
\end{equation*}
where $C$ is the constant of the integration.

Multiplying both sides by $e^t$ to solve for $x_2(t)$:
\begin{equation*}
    x_2(t) = -t-3 + Ce^t\,.
\end{equation*}

Now, use the initial condition $x_2(0) =1 $, we get
\begin{equation*}
    C = 4\,.
\end{equation*}
Therefore, 
\begin{equation*}
    x_2(t) = 4e^t -t-3\,.
\end{equation*}

Whereas under matrix $M'$: 
\begin{equation*}
\begin{split}
 \dot{x}_2(t) &= x_1(t) + z_1(t) + z_2(t) \\  
 &= t + 3\,.
\end{split}
\end{equation*}
Simple calculations show that 
\begin{equation*}
    \boldsymbol{x}_2(t) = t^2/2 + 3t + 1\,.
\end{equation*}
\newpage
\section{Detailed proofs}

\subsection{Proof of Theorem \ref{thm:identifiability ODE1}}\label{proof: ODE1}

\begin{proof}
Recall that the first derivative of $\boldsymbol{x}(t)$ can be expressed as:
\begin{equation*}
\begin{split}
    \dot{\boldsymbol{x}}(t) &= A \boldsymbol{x}(t) + B \boldsymbol{z}(t)\\
    &= A \boldsymbol{x}(t) + \sum_{k=0}^{r} \cfrac{B\boldsymbol{v}_k}{k+1} t^{k+1} + B\boldsymbol{z}_0\,.    
\end{split}
\end{equation*}
Set 
\begin{equation*}
    \boldsymbol{y}(t) = \begin{bmatrix}
        \boldsymbol{x}(t)\\
        1\\
        t\\
        t^2\\
        \vdots\\
        t^{r+1}
    \end{bmatrix}\,,
\end{equation*}
we see that $\boldsymbol{y}(t) \in \mathbb{R}^{d+r+2}$, and the first derivative of $\boldsymbol{y}(t)$ w.r.t. time $t$ can be expressed as 
\begin{equation*}
\begin{split}
    \dot{\boldsymbol{y}}(t) &= \begin{bmatrix}
        \dot{\boldsymbol{x}}(t) \\[2pt]
        0\\
        1\\
        2t\\
        \vdots\\
        (r+1)t^{r}
    \end{bmatrix}\\
    &=\underbrace{\begin{bmatrix}
        A & B \boldsymbol{z}_0 & B\boldsymbol{v}_0 & \frac{B\boldsymbol{v}_1} {2} & \ldots & \frac{B \boldsymbol{v}_{r-1}} {r} & \frac{B\boldsymbol{v}_r} {r+1} \\[2pt]
        \boldsymbol{0}_d & 0 & 0 & 0 & \ldots & 0 & 0\\
        \boldsymbol{0}_d & 1 & 0 & 0 & \ldots & 0 & 0\\
        \boldsymbol{0}_d & 0 & 2 & 0 & \ldots & 0 & 0\\
        \vdots & \vdots & \vdots & \vdots & \ddots & \vdots & \vdots\\
        \boldsymbol{0}_d & 0 & 0 & 0 & \ldots & r+1 & 0    
    \end{bmatrix}}_\textrm{denoted as $F$}
    \underbrace{\begin{bmatrix}
        \boldsymbol{x}(t)\\[2pt]
        1\\
        t\\
        t^2\\
        \vdots\\
        t^{r+1}
    \end{bmatrix}}_{\boldsymbol{y}(t)}\,,   
\end{split}
\end{equation*}
where $\boldsymbol{0}_d$ denotes a $d$ dimensional zero row vector. Obviously, 
\begin{equation*}
    \boldsymbol{y}(0) = [\boldsymbol{x}_0^T, 1, 0, 0, \ldots, 0]^{\top}\,,
\end{equation*}
we denote it as $\boldsymbol{y}_0$. Therefore, $\boldsymbol{y}(t)$ follows a homogeneous linear ODE system that can be expressed as:
\begin{equation}\label{eq:ODE y1}
\begin{split}
    \dot{\boldsymbol{y}}(t) &= F\boldsymbol{y}(t)\,,\\
    \boldsymbol{y}(0) &= \boldsymbol{y}_0\,,
\end{split}
\end{equation}
where $F \in \mathbb{R}^{(d+r+2)\times(d+r+2)}$. Worth noting that all state variables in the ODE system \eqref{eq:ODE y1} are observable.
Then according to Lemma \ref{lemma:identifiability ODE}, the identifiability of the dynamical system described by the ODE system \eqref{eq:ODE y1} is contingent upon the linear independence of the vectors $\{\boldsymbol{y}_0, F\boldsymbol{y}_0, F^2 \boldsymbol{y}_0, \ldots, F^{d+r+1}\boldsymbol{y}_0\}$. Specifically, the system is $(\boldsymbol{y}_0, F)$-identifiable if and only if this set of vectors is linearly independent, indicating that the matrix formed by these vectors, denoted by $H$, has a rank of $d+r+2$. In the following, we will elucidate that if and only assumption \textbf{A1} is satisfied, the rank of this matrix $H$ equals $d+r+2$.

Some calculations show that, 
\begin{equation}\label{eq:fky11}
    F^k \boldsymbol{y}_0 = \begin{bmatrix}
        A^{k-1}(A\boldsymbol{x}_0 + B\boldsymbol{z}_0)+ \sum_{j=0}^{k-2}j!A^{k-2-j} B \boldsymbol{v}_j \\
        0\\
        \vdots\\
        0\\
        k!\\
        0\\
        \vdots\\
        0
    \end{bmatrix} \ \ \ \mbox{for } k = 1, 2, \ldots, r+1\,,
\end{equation}
where $k!$ is the $(d+k+1)$-th element. 

And
\begin{equation}\label{eq:fky21}
    F^k \boldsymbol{y}_0 = \begin{bmatrix}
        A^{k-(r+2)}(A^{r+2}\boldsymbol{x}_0 + A^{r+1}B\boldsymbol{z}_0 + \sum_{j=0}^{r}j!A^{r-j} B\boldsymbol{v}_j)\\
        0\\
        \vdots\\
        0
    \end{bmatrix}\ \ \ \mbox{for } k = r+2,\ldots, r+d+1\,.
\end{equation}
According to assumption \textbf{A1} in Theorem \ref{thm:identifiability ODE1},
\begin{equation*}
   \boldsymbol{\beta} = A^{r+2}\boldsymbol{x}_0 + A^{r+1}B\boldsymbol{z}_0 + \sum_{j=0}^{r}j!A^{r-j} B\boldsymbol{v}_j\,,
\end{equation*}
therefore, $F^k\boldsymbol{y}_0$ can also be expressed as
\begin{equation}\label{eq:fky22}
    F^k \boldsymbol{y}_0 = \begin{bmatrix}
        A^{k-(r+2)}\boldsymbol{\beta}\\
        0\\
        \vdots\\
        0
    \end{bmatrix}\ \ \ \mbox{for } k = r+2,\ldots, r+d+1\,.
\end{equation}
We denote the matrix
\begin{equation*}
\begin{split}
        H :&= \begin{bmatrix}
        \boldsymbol{y}_0 & F\boldsymbol{y}_0 & F^2 \boldsymbol{y}_0 & \ldots & F^{r+1}\boldsymbol{y}_0 & F^{r+2} \boldsymbol{y}_0 & \ldots & F^{d+r+1}\boldsymbol{y}_0
    \end{bmatrix}\\
    :&= \begin{bmatrix}
        H_{11} & H_{12}\\
        H_{21} & H_{22}
    \end{bmatrix}
\end{split}
\end{equation*}
as a block matrix. Then, based on Equations \eqref{eq:fky11} and \eqref{eq:fky22}, one obtains that
\begin{equation*}
\begin{split}
    H_{11} &=\colvec{
        \boldsymbol{x}_0    &    A\boldsymbol{x}_0+B\boldsymbol{z}_0  &   A^2\boldsymbol{x}_0+AB\boldsymbol{z}_0+B\boldsymbol{v}_0 & \ldots  & A^{r+1}\boldsymbol{x}_0 + A^rB\boldsymbol{z}_0 + \sum_{j=0}^{r-1}j!A^{r-1-j} B \boldsymbol{v}_j\\
    }   \\
    &\in \mathbb{R}^{d\times (r+2)}\,,\\
    H_{12} &= \colvec{
        \boldsymbol{\beta} & A\boldsymbol{\beta} &\ldots & A^{d-1}\boldsymbol{\beta}\\ 
        }\\
    &\in \mathbb{R}^{d\times d}\,,   \\ 
    H_{21} &=\begin{bmatrix}
        1                   &    0                                    &   0                                                         & \ldots  & 0                  \\
        0                   &    1                                    &   0                                                         & \ldots  & 0                    \\
        0                   &    0                                    &   2!                                                        & \ldots  & 0                   \\
        \vdots              &    \vdots                               &   \vdots                                                    & \ddots  & \vdots                    \\
        0                   &    0                                    &   0                                                         & \ldots  & (r+1)!                    \\
    \end{bmatrix}\in \mathbb{R}^{(r+2)\times (r+2)}\,, \\
    H_{22} &= \boldsymbol{0}_{(r+2)\times d} \in \mathbb{R}^{(r+2)\times d}\,.
\end{split}
\end{equation*}

Some calculations show that 
\begin{equation*}
    \text{rank}(H) = \text{rank}(H_{12}) + \text{rank}(H_{21})\,.
\end{equation*}
It is apparent that 
\begin{equation*}
    \text{rank}(H_{21}) = r+2\,.
\end{equation*}
To achieve $\text{rank}(H) = d+r+2$, the rank of $H_{12}$ must be $d$. The rank of $H_{12}$ equals $d$ if and only if the set of vectors $\{\boldsymbol{\beta}, A\boldsymbol{\beta}, \ldots, A^{d-1}\boldsymbol{\beta} \}$ is linearly independent, that is, assumption \textbf{A1} is satisfied.

Now that we have proved that the ODE system \eqref{eq:ODE y1} is $(\boldsymbol{y}_0, F)$-identifiable if and only if assumption \textbf{A1} is satisfied. That is, under assumption \textbf{A1}, the trajectory 
$\boldsymbol{y}(\cdot; \boldsymbol{y}_0, F)$ uniquely determines both $\boldsymbol{y}_0$ and matrix $F$. Consequently, it also uniquely determines $(\boldsymbol{x}_0, A, B\boldsymbol{z}_0, B\boldsymbol{v}_0, \ldots, B\boldsymbol{v}_r )$, thus establishing that the ODE system \eqref{eq:ODE1} is $\boldsymbol{\theta}$-identifiable if and only if assumption \textbf{A1} is satisfied. 
\end{proof}
\subsection{Proof of Theorem \ref{thm: identifiability ODE2 continuous}}\label{proof: ODE2 continuous}

\begin{proof}
Recall that the first derivative of $\boldsymbol{x}(t)$ can be expressed as:
\begin{equation*}
\begin{split}
    \dot{\boldsymbol{x}}(t) &= A \boldsymbol{x}(t) + B \boldsymbol{z}(t)\\
    &= A \boldsymbol{x}(t) + \sum_{k=0}^{p-1} \cfrac{BG^k\boldsymbol{z}_0}{k!}t^k\,.
\end{split}
\end{equation*}
Set 
\begin{equation*}
    \boldsymbol{y}(t) = \begin{bmatrix}
        \boldsymbol{x}(t)\\
        1\\
        t\\
        t^2\\
        \vdots\\
        t^{p-1}
    \end{bmatrix}\,,
\end{equation*}
we see that $\boldsymbol{y}(t) \in \mathbb{R}^{d+p}$, and the first derivative of $\boldsymbol{y}(t)$ w.r.t. time $t$ can be expressed as 
\begin{equation*}
\begin{split}
    \dot{\boldsymbol{y}}(t) &= \begin{bmatrix}
        \dot{\boldsymbol{x}}(t) \\[2pt]
        0\\
        1\\
        2t\\
        \vdots\\
        (p-1)t^{p-2}
    \end{bmatrix}\\
    &=\underbrace{\begin{bmatrix}
        A & B \boldsymbol{z}_0 & BG\boldsymbol{z}_0 & \frac{BG^2 \boldsymbol{z}_0} {2!} & \ldots & \frac{BG^{p-2} \boldsymbol{z}_0} {(p-2)!} & \frac{BG^{p-1} \boldsymbol{z}_0} {(p-1)!} \\[2pt]
        \boldsymbol{0}_d & 0 & 0 & 0 & \ldots & 0 & 0\\
        \boldsymbol{0}_d & 1 & 0 & 0 & \ldots & 0 & 0\\
        \boldsymbol{0}_d & 0 & 2 & 0 & \ldots & 0 & 0\\
        \vdots & \vdots & \vdots & \vdots & \ddots & \vdots & \vdots\\
        \boldsymbol{0}_d & 0 & 0 & 0 & \ldots & p-1 & 0    
    \end{bmatrix}}_\textrm{denoted as $F$}
    \underbrace{\begin{bmatrix}
        \boldsymbol{x}(t)\\[2pt]
        1\\
        t\\
        t^2\\
        \vdots\\
        t^{p-1}
    \end{bmatrix}}_{\boldsymbol{y}(t)}\,,    
\end{split}
\end{equation*}
where $\boldsymbol{0}_d$ denotes a $d$ dimensional zero row vector. Obviously, 
\begin{equation*}
    \boldsymbol{y}(0) = [\boldsymbol{x}_0^T, 1, 0, 0, \ldots, 0]^{\top}\,,
\end{equation*}
we denote it as $\boldsymbol{y}_0$. Therefore, $\boldsymbol{y}(t)$ follows a homogeneous linear ODE system that can be expressed as:
\begin{equation}\label{eq:ODE y}
\begin{split}
    \dot{\boldsymbol{y}}(t) &= F\boldsymbol{y}(t)\,,\\
    \boldsymbol{y}(0) &= \boldsymbol{y}_0\,,
\end{split}
\end{equation}
where $F \in \mathbb{R}^{(d+p)\times(d+p)}$. Worth noting that all state variables in the ODE system \eqref{eq:ODE y} are observable.
Then according to Lemma \ref{lemma:identifiability ODE}, the identifiability of the dynamical system described by the ODE system \eqref{eq:ODE y} is contingent upon the linear independence of the vectors $\{\boldsymbol{y}_0, F\boldsymbol{y}_0, F^2 \boldsymbol{y}_0, \ldots, F^{d+p-1}\boldsymbol{y}_0\}$. Specifically, the system is $(\boldsymbol{y}_0, F)$-identifiable if and only if this set of vectors is linearly independent, indicating that the matrix formed by these vectors, denoted by $H$, has a rank of $d+p$. In the following, we will elucidate that if and only assumption \textbf{B1} is satisfied, the rank of this matrix $H$ equals $d+p$.

Some calculations show that, 
\begin{equation}\label{eq:fky1}
    F^k \boldsymbol{y}_0 = \begin{bmatrix}
        A^k\boldsymbol{x}_0 + \sum_{j=0}^{k-1}A^{k-1-j} B G^j \boldsymbol{z}_0\\
        0\\
        \vdots\\
        0\\
        k!\\
        0\\
        \vdots\\
        0
    \end{bmatrix} \ \ \ \mbox{for } k = 1, 2, \ldots, p-1\,,
\end{equation}
where $k!$ is the $(d+k+1)$-th element. 

And
\begin{equation}\label{eq:fky2}
    F^k \boldsymbol{y}_0 = \begin{bmatrix}
        A^{k-p}(A^p\boldsymbol{x}_0 + \sum_{j=0}^{p-1}A^{p-1-j} B G^j \boldsymbol{z}_0)\\
        0\\
        \vdots\\
        0
    \end{bmatrix}\ \ \ \mbox{for } k = p, p+1, \ldots, p+d-1\,.
\end{equation}
According to assumption \textbf{B1} in Theorem \ref{thm:identifiability ODE1},
\begin{equation*}
    \boldsymbol{\gamma} = A^p\boldsymbol{x}_0 + \sum_{j=0}^{p-1}A^{p-1-j} B G^j \boldsymbol{z}_0\,,
\end{equation*}
therefore, $F^k\boldsymbol{y}_0$ can also be expressed as
\begin{equation}\label{eq:fky3}
    F^k \boldsymbol{y}_0 = \begin{bmatrix}
        A^{k-p}\boldsymbol{\gamma}\\
        0\\
        \vdots\\
        0
    \end{bmatrix}\ \ \ \mbox{for } k = p, p+1, \ldots, p+d-1\,.
\end{equation}
We denote the matrix
\begin{equation*}
\begin{split}
        H :&= \begin{bmatrix}
        \boldsymbol{y}_0 & F\boldsymbol{y}_0 & F^2 \boldsymbol{y}_0 & \ldots & F^{p-1}\boldsymbol{y}_0 & F^p \boldsymbol{y}_0 & \ldots & F^{p+d-1}\boldsymbol{y}_0
    \end{bmatrix}\\
    :&= \begin{bmatrix}
        H_{11} & H_{12}\\
        H_{21} & H_{22}
    \end{bmatrix}
\end{split}
\end{equation*}
as a block matrix. Then, based on Equations \eqref{eq:fky1} and \eqref{eq:fky3}, one obtains that
\begin{equation*}
\begin{split}
    H_{11} &=\colvec{
        \boldsymbol{x}_0    &    A\boldsymbol{x}_0+B\boldsymbol{z}_0  &   A^2\boldsymbol{x}_0+AB\boldsymbol{z}_0+BG\boldsymbol{z}_0 & \ldots  & A^{p-1}\boldsymbol{x}_0 + \sum_{j=0}^{p-2}A^{p-2-j} B G^j \boldsymbol{z}_0\\
    }   \\
    &\in \mathbb{R}^{d\times p}\,,\\
    H_{12} &= 
     \colvec{
        \boldsymbol{\gamma} & A\boldsymbol{\gamma} &\ldots & A^{d-1}\boldsymbol{\gamma}\\
    }\\
    &\in \mathbb{R}^{d\times d}\,,   \\ 
    H_{21} &=\begin{bmatrix}
        1                   &    0                                    &   0                                                         & \ldots  & 0                  \\
        0                   &    1                                    &   0                                                         & \ldots  & 0                    \\
        0                   &    0                                    &   2!                                                        & \ldots  & 0                   \\
        \vdots              &    \vdots                               &   \vdots                                                    & \ddots  & \vdots                    \\
        0                   &    0                                    &   0                                                         & \ldots  & (p-1)!                    \\
    \end{bmatrix}\in \mathbb{R}^{p\times p}\,, \\
    H_{22} &= \boldsymbol{0}_{p\times d} \in \mathbb{R}^{p\times d}\,.
\end{split}
\end{equation*}

Some calculations show that 
\begin{equation*}
    \text{rank}(H) = \text{rank}(H_{12}) + \text{rank}(H_{21})\,.
\end{equation*}
It is apparent that 
\begin{equation*}
    \text{rank}(H_{21}) = p\,.
\end{equation*}
To achieve $\text{rank}(H) = d+p$, the rank of $H_{12}$ must be $d$. The rank of $H_{12}$ equals $d$ if and only if the set of vectors $\{\boldsymbol{\gamma}, A\boldsymbol{\gamma}, \ldots, A^{d-1}\boldsymbol{\gamma} \}$ is linearly independent, that is, assumption \textbf{B1} is satisfied.

Now that we have proved that the ODE system \eqref{eq:ODE y} is $(\boldsymbol{y}_0, F)$-identifiable if and only if assumption \textbf{B1} is satisfied. That is, under assumption \textbf{B1}, the trajectory $\boldsymbol{y}(\cdot; \boldsymbol{y}_0, F)$ uniquely determines both
$\boldsymbol{y}_0$ and the matrix $F$. Consequently, it also uniquely determines $(\boldsymbol{x}_0, A, B\boldsymbol{z}_0, BG\boldsymbol{z}_0, \ldots, BG^{p-1}\boldsymbol{z}_0 )$, thus establishing that the ODE system \eqref{eq:ODE2} is $\boldsymbol{\eta}$-identifiable if and only if assumption \textbf{B1} is satisfied. 
\end{proof}

\subsection{Proof of Theorem \ref{thm: identifiability ODE2 discrete}}\label{proof: ODE2 discrete}
Before providing the main proof, we first present two lemmas.
\begin{lemma}\cite[Theorem 3.4]{stanhope2014identifiability}\label{lemma:trajectory}
The ODE system \eqref{eq:ODE} is $(\boldsymbol{x}_0, A)$-identifiable if and only if the trajectory $\boldsymbol{x}(\cdot; \boldsymbol{x}_0, A)$ is not confined to a proper subspace of $\mathbb{R}^d$.
\end{lemma}
\begin{lemma}\cite[Lemma 6.1]{stanhope2014identifiability}\label{lemma:discrete}
    Trajectory $\boldsymbol{x}(\cdot; \boldsymbol{x}_0, A)$ is not confined to a proper subspace of $\mathbb{R}^d$ if and only if there exists $t_1, t_2, \ldots, t_d$ such that $\boldsymbol{x}_1, \boldsymbol{x}_2, \ldots, \boldsymbol{x}_d$ are linearly independent.
\end{lemma}
\begin{proof}
\begin{repairblock}
Let $m=d+p$. Using the augmentation in Appendix \ref{proof: ODE2 continuous}, write
\[
\dot{\boldsymbol{y}}(t)=F\boldsymbol{y}(t),\qquad
F=\begin{bmatrix}A&C\\0&J_p\end{bmatrix},\qquad
C=\left[\frac{BG^k\boldsymbol{z}_0}{k!}\right]_{k=0}^{p-1},
\]
where the only nonzero entries of $J_p$ are $(J_p)_{k+1,k}=k$, $k=1,\ldots,p-1$.
Thus $\sigma(F)=\sigma(A)\cup\{0\}$. Under \textbf{C0}, put $M=e^{\Delta F}$; then $\boldsymbol{y}_j=M^{j-1}\boldsymbol{y}_1$.

Suppose \textbf{C1} holds. The first $m$ vectors of this sequence must be linearly independent: a first dependence would make their preceding span invariant under $M$, confining all subsequent vectors to that proper subspace. Since $n\ge m+1$, the observations therefore determine
\[
M=Y_+Y_-^{-1},\qquad
Y_-=[\boldsymbol{y}_1,\ldots,\boldsymbol{y}_m],\qquad
Y_+=[\boldsymbol{y}_2,\ldots,\boldsymbol{y}_{m+1}].
\]
Every competing parameter tuple with the same observations has an augmented generator $F'$ satisfying $e^{\Delta F'}=M$. Both $\Delta F$ and $\Delta F'$ have spectra in $\{\lambda\in\mathbb C:-\pi<\operatorname{Im}\lambda<\pi\}$. Uniqueness of the principal matrix logarithm\footnote{\repair{Nicholas J. Higham, \emph{Functions of Matrices: Theory and Computation}, Society for Industrial and Applied Mathematics, 2008.}} gives
\[
F=F'=\Delta^{-1}\operatorname{Log}M,\qquad
\boldsymbol{y}_0=e^{-t_1F}\boldsymbol{y}_1.
\]
Hence $(\boldsymbol{x}_0,A,B\boldsymbol{z}_0,BG\boldsymbol{z}_0,\ldots,BG^{p-1}\boldsymbol{z}_0)$ is uniquely determined, proving sufficiency.

Conversely, if \textbf{C1} fails, the space
$K=\operatorname{span}\{M^k\boldsymbol{y}_1:k\ge0\}$ is a proper subspace: its first $m$ Krylov vectors already span it. Since $K$ is $M$-invariant and $\operatorname{Log}M$ is a polynomial in $M$, $K$ is also $F$-invariant. Consequently, $\boldsymbol{y}(t)=e^{(t-t_1)F}\boldsymbol{y}_1\in K$ for every $t\ge0$. Choose nonzero $\boldsymbol{w}\in K^\perp$ and nonzero $\boldsymbol{u}\in\mathbb R^d$, and set
\[
F_\varepsilon=F+\varepsilon\begin{bmatrix}\boldsymbol{u}\\\boldsymbol{0}_{p\times1}\end{bmatrix}\boldsymbol{w}^\top
=\begin{bmatrix}A_\varepsilon&C_\varepsilon\\0&J_p\end{bmatrix}.
\]
For sufficiently small nonzero real $\varepsilon$, $A_\varepsilon\in\mathcal A_\Delta$ because the spectral restriction is open. Moreover, $F_\varepsilon\boldsymbol{y}(t)=F\boldsymbol{y}(t)$, so the same trajectory, with the same initial value, solves the perturbed system.

This perturbation is realizable within the latent DAG model. Write $C_\varepsilon=[\boldsymbol{c}_0^\varepsilon,\ldots,\boldsymbol{c}_{p-1}^\varepsilon]$, let $N_p$ have ones on its first superdiagonal and zeros elsewhere, and let $\boldsymbol{e}_j$ be the $j$th standard basis vector of $\mathbb R^p$. Choose
\[
G_\varepsilon=N_p,\qquad \boldsymbol{z}_{0,\varepsilon}=\boldsymbol{e}_p,\qquad
B_\varepsilon \boldsymbol{e}_{p-k}=k!\boldsymbol{c}_k^\varepsilon\quad(k=0,\ldots,p-1).
\]
Then $G_\varepsilon$ is strictly upper triangular and $B_\varepsilon G_\varepsilon^k\boldsymbol{z}_{0,\varepsilon}/k!=\boldsymbol{c}_k^\varepsilon$. Thus an admissible competing tuple produces the same entire observable trajectory. Since $F_\varepsilon\ne F$, at least one of $A$ and the products $BG^k\boldsymbol{z}_0$ differs. This proves necessity.
\end{repairblock}
\end{proof}

\subsection{Proof of Theorem \ref{thm: identifiability ODE2 more}}\label{proof: ODE2 more}
\begin{proof}

Under assumption \textbf{B2}, since each $\boldsymbol{z}_0^{*i}$ satisfies assumption \textbf{B1}, Theorem \ref{thm: identifiability ODE2 continuous} implies that the ODE system \eqref{eq:ODE2} is $\boldsymbol{\eta}_i$-identifiable for all $i=1, \ldots, p$. That is, one can identify 
\begin{equation*}
   (\boldsymbol{x}_0, A, B\boldsymbol{z}_0^{*i}, BG\boldsymbol{z}_0^{*i}, \ldots, BG^{p-1}\boldsymbol{z}_0^{*i} ) 
\end{equation*}
 for all $i=1, \ldots, p$.

Next, we will prove that matrix $B$ is identifiable under assumption \textbf{B3}. 

Define the matrix
\begin{equation*}
S :=\begin{bmatrix}
    B\boldsymbol{z}_0^{*1} & B\boldsymbol{z}_0^{*2} & \ldots & B\boldsymbol{z}_0^{*p}
\end{bmatrix}\,,   
\end{equation*}
we know that $S \in \mathbb{R}^{d\times p}$, and $S$ is identifiable. The matrix $S$ can also be expressed as:
\begin{equation*}
\begin{split}
S &= B \begin{bmatrix}
    \boldsymbol{z}_0^{*1} & \boldsymbol{z}_0^{*2} & \ldots & \boldsymbol{z}_0^{*p}
\end{bmatrix}\\
:&= B Z\,,
\end{split}
\end{equation*}
where under assumption \textbf{B3}, the matrix $Z$ is invertible. Therefore, 
\begin{equation*}
   B = SZ^{-1} \,.
\end{equation*}
Since $Z$ is a known matrix, $B$ is identifiable.

Similarly, we can prove that $BG^j$ for $j=1, \ldots, p-1$ is also identifiable.

We now show that, under assumption \textbf{B4}, the matrix $G$ is identifiable.

Define the matrix
\begin{equation*}
    W := \begin{bmatrix}
        B\\
        BG\\
        \vdots \\
        BG^{p-1}
    \end{bmatrix}\,,
\end{equation*}
we know that $W \in \mathbb{R}^{{dp}\times p}$, and $W$ is identifiable.

Since $G$ is a $p \times p $ nilpotent matrix, $G^p = \boldsymbol{0}$, thus $BG^p = \boldsymbol{0}$ . If we define the matrix
\begin{equation*}
    V := \begin{bmatrix}
        BG\\
        BG^2\\
        \vdots \\
        BG^p
    \end{bmatrix}\,,
\end{equation*}
then $V \in \mathbb{R}^{dp\times p}$, and $V$ is identifiable. The matrix $V$ can also be expressed as:
\begin{equation}\label{eq:system of linear equations 3}
    V = \begin{bmatrix}
        B\\
        BG\\
        \vdots \\
        BG^{p-1}
    \end{bmatrix}G = WG\,.
\end{equation}
Under assumption \textbf{B4}, one can find $p$ linearly independent rows in matrix $W$. Denote the matrix composed of these $p$ linearly independent rows as $W_p$, which is invertible. Denote the matrix composed of the corresponding $p$ rows of $V$ as $V_p$, we have
\begin{equation*}
    V_p = W_p G\,.
\end{equation*}
Since $W_p$ is invertible, then 
\begin{equation*}
    G =  W_p^{-1}V_p\,.
\end{equation*}
Because both $V_p$ and $W_p$ are identifiable, $G$ is also identifiable.
\end{proof}

\subsection{Proof of Theorem \ref{thm: identifiability ODE2 discrete more}}\label{proof: ODE2 discrete more}
\begin{proof}
Under assumption \textbf{C2}, for each $i \in \{1, \ldots, p\}$, the corresponding observations satisfy assumption \textbf{C1}. Based on Theorem \ref{thm: identifiability ODE2 discrete}, the ODE system \eqref{eq:ODE2} is $\boldsymbol{\eta}_i$-identifiable for all $i=1, \ldots, p$. This implies that one can identify 
\begin{equation*}
    (\boldsymbol{x}_0, A, B\boldsymbol{z}_0^{*i}, BG\boldsymbol{z}_0^{*i}, \ldots, BG^{p-1}\boldsymbol{z}_0^{*i} )
\end{equation*}
for all $i=1, \ldots, p$.

According to the proof of Theorem \ref{thm: identifiability ODE2 more}, under assumptions \textbf{B3} and \textbf{B4}, matrices $B$ and $G$ are also identifiable.
\end{proof}

\begin{repairblock}
\subsection{Correction to the discrete-observation results}\label{app:discrete-correction}
The original discrete-observation claims in Theorems~\ref{thm: identifiability ODE2 discrete}, \ref{thm: identifiability ODE2 discrete more}, and \ref{thm: identifiability ODE2 discrete single intervene} were too broad: the argument did not establish identifiability from the stated finite samples. The revised results impose the sampling condition \textbf{C0} and restrict both the true and competing matrices $A$ to $\mathcal A_\Delta$. The corresponding definitions have been updated, and Appendix~\ref{proof: ODE2 discrete} provides a replacement proof using the unique principal matrix logarithm. The rank conditions \textbf{C1}--\textbf{C3} and the continuous-trajectory theorem statements are unchanged.
\end{repairblock}

\newpage
\section{Identifiability conditions of the linear ODE system (\ref{eq:ODE1}) with other \textit{f(t)}}\label{app:more cases of ODE1}

In this section, we provide identifiability conditions for the linear ODE system \eqref{eq:ODE1} with $\boldsymbol{f}(t) = \boldsymbol{v}e^t$ and $\boldsymbol{f}(t) = \boldsymbol{v}_1sin(t) + \boldsymbol{v}_2cos(t)$. For notational simplicity, we slightly abuse notation by using the same symbols as in Section \ref{sec:ODE1}.

\subsection{When \textit{f(t)} follows an exponenial function of time \textit{t} }
We define $\boldsymbol{f}(t)$ in the ODE system \eqref{eq:ODE1} as:
\begin{equation*}
    \boldsymbol{f}(t) = \boldsymbol{v}e^t\,, \ \ \ \boldsymbol{v}\in \mathbb{R}^p\,.
\end{equation*}
Simple calculations show that
\begin{equation*}
    \boldsymbol{z}(t) = \boldsymbol{v}e^t + \boldsymbol{z}_0 - \boldsymbol{v}\,.
\end{equation*}
Thus,
\begin{equation}\label{eq:dxt et}
\begin{split}
    \dot{\boldsymbol{x}}(t) &= A\boldsymbol{x}(t) + B \boldsymbol{z}(t)  \\
    &= A \boldsymbol{x}(t) + B\boldsymbol{v}e^t + B\boldsymbol{z}_0 - B\boldsymbol{v}\,.
\end{split}
\end{equation}

We denote the unknown parameters of the ODE system \eqref{eq:ODE1} with this $\boldsymbol{f}(t)$ as $\boldsymbol{\theta}$, specifically, $\boldsymbol{\theta}:=(\boldsymbol{x}_0,\boldsymbol{z}_0, A, B, \boldsymbol{v})$. Let $[\boldsymbol{x}^T(t; \boldsymbol{\theta}), \boldsymbol{z}^T(t; \boldsymbol{\theta})]^{T}$ denote the solution of the ODE system \eqref{eq:ODE1}. It is important to note that under our hidden variables setting, only $\boldsymbol{x}(t; \boldsymbol{\theta})$ is observable. Based on Equation \eqref{eq:dxt et}, we present the following identifiability definition.

\begin{definition}\label{def:identifiability ODE1 et}
For $\boldsymbol{x}_0 \in \mathbb{R}^d, \boldsymbol{z}_0 \in \mathbb{R}^p, A\in \mathbb{R}^{d\times d}, B\in \mathbb{R}^{d\times p}$ and $\boldsymbol{v} \in \mathbb{R}^{p}$, for all $\boldsymbol{x}'_0 \in \mathbb{R}^d$, all $\boldsymbol{z}'_0 \in \mathbb{R}^p$, all $A'\in \mathbb{R}^{d\times d}$, all $B' \in \mathbb{R}^{d\times p}$, and all $\boldsymbol{v}' \in \mathbb{R}^{p}$, we denote $\boldsymbol{\theta}':=(\boldsymbol{x}'_0,\boldsymbol{z}'_0, A', B', \boldsymbol{v}')$, we say the ODE system \eqref{eq:ODE1} is
 $\boldsymbol{\theta}$-identifiable:
if $(\boldsymbol{x}_0, A, B\boldsymbol{z}_0, B\boldsymbol{v}) \neq (\boldsymbol{x}'_0, A', B'\boldsymbol{z}'_0, B'\boldsymbol{v}')$, it holds that $\boldsymbol{x}(\cdot; \boldsymbol{\theta})\ {\neq} \ \boldsymbol{x}(\cdot; \boldsymbol{\theta}') $.
\end{definition}

According to Definition \ref{def:identifiability ODE1 et}, if the ODE system \eqref{eq:ODE1} with an exponential $\boldsymbol{f}(t)$ is $\boldsymbol{\theta}$-identifiable, then the trajectory of the system can uniquely determine the values of $(\boldsymbol{x}_0, A, B\boldsymbol{z}_0, B\boldsymbol{v})$. This determination is sufficient to identify the causal relationships between observable variables $\boldsymbol{x}$ as described by Equation \eqref{eq:dxt et}. Consequently, one can safely intervene in the observable variables of the ODE system and make reliable causal inferences, despite the fact that matrix $B$ cannot be identified under this definition.

\begin{theorem}\label{thm:identifiability ODE1 et}
    For $\boldsymbol{x}_0 \in \mathbb{R}^d, \boldsymbol{z}_0 \in \mathbb{R}^p, A\in \mathbb{R}^{d\times d}, B\in \mathbb{R}^{d\times p}$, and $\boldsymbol{v} \in \mathbb{R}^{p}$, the ODE system \eqref{eq:ODE1} is $\boldsymbol{\theta}$-identifiable if and only if assumption \textbf{D1} is satisfied.
    \begin{enumerate}
        \item [\textbf{D1}] the set of vectors $\{\boldsymbol{y}_0, F\boldsymbol{y}_0, \ldots, F^{d+1}\boldsymbol{y}_0 \}$ is linearly independent, where $\boldsymbol{y}_0 =[\boldsymbol{x}_0^T, 1, 1]^T$, and
        \begin{equation*}
          F=  \begin{bmatrix}
        A & B\boldsymbol{v} & B\boldsymbol{z}_0 - B\boldsymbol{v}\\
        \boldsymbol{0}_d & 1 & 0 \\
        \boldsymbol{0}_d & 0 & 0
    \end{bmatrix}\,,
        \end{equation*} $\boldsymbol{0}_d$ denotes a $d$ dimensional zero row vector.
    \end{enumerate}
\end{theorem}

The proof of Theorem \ref{thm:identifiability ODE1 et} is presented below. Condition \textbf{D1} is both sufficient and necessary, indicating, from a geometric perspective, that the vector $\boldsymbol{y}_0$ is not contained in an $F$-invariant proper subspace of $\mathbb{R}^{d+2}$.

\begin{proof}
Set
\begin{equation*}
    \boldsymbol{y}(t) = \begin{bmatrix}
        \boldsymbol{x}(t)\\
        e^t\\
        1
    \end{bmatrix}\,,
\end{equation*}
we see that $\boldsymbol{y}(t) \in \mathbb{R}^{d+2}$, and the first derivative of $\boldsymbol{y}(t)$ w.r.t. time $t$ can be expressed as
\begin{equation*}
    \dot{\boldsymbol{y}}(t) = \begin{bmatrix}
        \dot{\boldsymbol{x}}(t)\\
        e^t\\
        0
    \end{bmatrix} = \underbrace{\begin{bmatrix}
        A & B\boldsymbol{v} & B\boldsymbol{z}_0 - B\boldsymbol{v}\\
        \boldsymbol{0}_d & 1 & 0 \\
        \boldsymbol{0}_d & 0 & 0
    \end{bmatrix}}_\textrm{$F$} \underbrace{\begin{bmatrix}
        \boldsymbol{x}(t)\\
        e^t\\
        1
    \end{bmatrix}}_{\boldsymbol{y}(t)}\,,
\end{equation*}
where $\boldsymbol{0}_d$ denotes a $d$ dimensional zero row vector. Obviously, 
\begin{equation*}
    \boldsymbol{y}(0) = [\boldsymbol{x}_0^T, 1, 1]^T = \boldsymbol{y}_0\,.
\end{equation*}
Therefore, $\boldsymbol{y}(t)$ follows a homogeneous linear ODE system that can be expressed as:
\begin{equation}\label{eq:ODE y et}
\begin{split}
    \dot{\boldsymbol{y}}(t) &= F\boldsymbol{y}(t)\,,\\
    \boldsymbol{y}(0) &= \boldsymbol{y}_0\,,
\end{split}
\end{equation}
where $F \in \mathbb{R}^{(d+2)\times(d+2)}$. Worth noting that all state variables in the ODE system \eqref{eq:ODE y et} are observable.
Then according to Lemma \ref{lemma:identifiability ODE}, the system \eqref{eq:ODE y et} is $(\boldsymbol{y}_0, F)$-identifiable if and only if condition \textbf{D1} stated in Theorem \ref{thm:identifiability ODE1 et} is satisfied. That is, under assumption \textbf{D1}, the trajectory $\boldsymbol{y}(\cdot; \boldsymbol{y}_0, F)$ uniquely determines both $\boldsymbol{y}_0$ and matrix $F$. Consequently, it also uniquely determines $(\boldsymbol{x}_0, A, B\boldsymbol{z}_0, B\boldsymbol{v})$, thus establishing that the ODE system \eqref{eq:ODE1} is $\boldsymbol{\theta}$-identifiable if and only if assumption \textbf{D1} is satisfied. 
\end{proof}

\subsection{When \textit{f(t)} follows an trigonometric function of time \textit{t} }

We define $\boldsymbol{f}(t)$ in the ODE system \eqref{eq:ODE1} as:
\begin{equation*}
    \boldsymbol{f}(t) = \boldsymbol{v}_1 sin(t) + \boldsymbol{v}_2 cos(t)\,, \ \ \ \boldsymbol{v}_1, \boldsymbol{v}_2\in \mathbb{R}^p\,.
\end{equation*}
Simple calculations show that
\begin{equation*}
    \boldsymbol{z}(t) = \boldsymbol{v}_2 sin(t) -\boldsymbol{v}_1cos(t) + \boldsymbol{z}_0 + \boldsymbol{v}_1\,.
\end{equation*}
Thus,
\begin{equation}\label{eq:dxt sint}
\begin{split}
    \dot{\boldsymbol{x}}(t) &= A\boldsymbol{x}(t) + B \boldsymbol{z}(t)  \\
    &= A \boldsymbol{x}(t) +  B\boldsymbol{v}_2 sin(t) - B\boldsymbol{v}_1cos(t) + B\boldsymbol{z}_0 + B\boldsymbol{v}_1\,.
\end{split}
\end{equation}

We denote the unknown parameters of the ODE system \eqref{eq:ODE1} with this $\boldsymbol{f}(t)$ as $\boldsymbol{\theta}$, specifically, $\boldsymbol{\theta}:=(\boldsymbol{x}_0,\boldsymbol{z}_0, A, B, \boldsymbol{v}_1, \boldsymbol{v}_2)$. Let $[\boldsymbol{x}^T(t; \boldsymbol{\theta}), \boldsymbol{z}^T(t; \boldsymbol{\theta})]^{T}$ denote the solution of the ODE system \eqref{eq:ODE1}. It is important to note that under our hidden variables setting, only $\boldsymbol{x}(t; \boldsymbol{\theta})$ is observable. Based on Equation \eqref{eq:dxt sint}, we present the following identifiability definition.

\begin{definition}\label{def:identifiability ODE1 sint}
For $\boldsymbol{x}_0 \in \mathbb{R}^d, \boldsymbol{z}_0 \in \mathbb{R}^p, A\in \mathbb{R}^{d\times d}, B\in \mathbb{R}^{d\times p}$ and $\boldsymbol{v}_1, \boldsymbol{v}_2 \in \mathbb{R}^{p}$, for all $\boldsymbol{x}'_0 \in \mathbb{R}^d$, all $\boldsymbol{z}'_0 \in \mathbb{R}^p$, all $A'\in \mathbb{R}^{d\times d}$, all $B' \in \mathbb{R}^{d\times p}$, and all $\boldsymbol{v}'_1, \boldsymbol{v}'_2 \in \mathbb{R}^{p}$, we denote $\boldsymbol{\theta}':=(\boldsymbol{x}'_0,\boldsymbol{z}'_0, A', B', \boldsymbol{v}'_1, \boldsymbol{v}'_2)$, we say the ODE system \eqref{eq:ODE1} is
 $\boldsymbol{\theta}$-identifiable:
if $(\boldsymbol{x}_0, A, B\boldsymbol{z}_0, B\boldsymbol{v}_1, B\boldsymbol{v}_2) \neq (\boldsymbol{x}'_0, A', B'\boldsymbol{z}'_0, B'\boldsymbol{v}'_1, B'\boldsymbol{v}'_2)$, it holds that $\boldsymbol{x}(\cdot; \boldsymbol{\theta})\ {\neq} \ \boldsymbol{x}(\cdot; \boldsymbol{\theta}') $.
\end{definition}

According to Definition \ref{def:identifiability ODE1 sint}, if the ODE system \eqref{eq:ODE1} with a trigonometric $\boldsymbol{f}(t)$ is $\boldsymbol{\theta}$-identifiable, then the trajectory of the system can uniquely determine the values of $(\boldsymbol{x}_0, A, B\boldsymbol{z}_0, B\boldsymbol{v}_1, B\boldsymbol{v}_2)$. This determination is sufficient to identify the causal relationships between observable variables $\boldsymbol{x}$ as described by Equation \eqref{eq:dxt sint}. Consequently, one can safely intervene in the observable variables of the ODE system and make reliable causal inferences, despite the fact that matrix $B$ cannot be identified under this definition.

\begin{theorem}\label{thm:identifiability ODE1 sint}
    For $\boldsymbol{x}_0 \in \mathbb{R}^d, \boldsymbol{z}_0 \in \mathbb{R}^p, A\in \mathbb{R}^{d\times d}, B\in \mathbb{R}^{d\times p}$, and $\boldsymbol{v}_1, \boldsymbol{v}_2 \in \mathbb{R}^{p}$, the ODE system \eqref{eq:ODE1} is $\boldsymbol{\theta}$-identifiable if and only if assumption \textbf{E1} is satisfied.
    \begin{enumerate}
        \item [\textbf{E1}] the set of vectors $\{\boldsymbol{y}_0, F\boldsymbol{y}_0, \ldots, F^{d+2}\boldsymbol{y}_0 \}$ is linearly independent, where $\boldsymbol{y}_0 =[\boldsymbol{x}_0^T, 0, 1, 1]^T$, and
        \begin{equation*}
          F=  \begin{bmatrix}
        A & B\boldsymbol{v}_2 & - B\boldsymbol{v}_1 &  B\boldsymbol{z}_0 +  B\boldsymbol{v}_1 \\
        \boldsymbol{0}_d & 0  & 1 & 0 \\
        \boldsymbol{0}_d & -1 & 0 & 0\\
        \boldsymbol{0}_d & 0  & 0 & 0
    \end{bmatrix}\,,
        \end{equation*} $\boldsymbol{0}_d$ denotes a $d$ dimensional zero row vector.
    \end{enumerate}
\end{theorem}

The proof of Theorem \ref{thm:identifiability ODE1 sint} is presented below. Condition \textbf{E1} is both sufficient and necessary, indicating, from a geometric perspective, that the vector $\boldsymbol{y}_0$ is not contained in an $F$-invariant proper subspace of $\mathbb{R}^{d+3}$.

\begin{proof}
Set
\begin{equation*}
    \boldsymbol{y}(t) = \begin{bmatrix}
        \boldsymbol{x}(t)\\
        sin(t)\\
        cos(t)\\
        1
    \end{bmatrix}\,,
\end{equation*}
we see that $\boldsymbol{y}(t) \in \mathbb{R}^{d+3}$, and the first derivative of $\boldsymbol{y}(t)$ w.r.t. time $t$ can be expressed as
\begin{equation*}
    \dot{\boldsymbol{y}}(t) = \begin{bmatrix}
        \dot{\boldsymbol{x}}(t)\\
        cos(t)\\
        -sin(t)\\
        0
    \end{bmatrix} = \underbrace{\begin{bmatrix}
        A & B\boldsymbol{v}_2 & - B\boldsymbol{v}_1 &  B\boldsymbol{z}_0 +  B\boldsymbol{v}_1 \\
        \boldsymbol{0}_d & 0  & 1 & 0 \\
        \boldsymbol{0}_d & -1 & 0 & 0\\
        \boldsymbol{0}_d & 0  & 0 & 0
    \end{bmatrix}}_\textrm{$F$} \underbrace{\begin{bmatrix}
        \boldsymbol{x}(t)\\
        sin(t)\\
        cos(t)\\
        1
    \end{bmatrix}}_{\boldsymbol{y}(t)}\,,
\end{equation*}
where $\boldsymbol{0}_d$ denotes a $d$ dimensional zero row vector. Obviously, 
\begin{equation*}
    \boldsymbol{y}(0) = [\boldsymbol{x}_0^T, 0, 1, 1]^T = \boldsymbol{y}_0\,.
\end{equation*}
Therefore, $\boldsymbol{y}(t)$ follows a homogeneous linear ODE system that can be expressed as:
\begin{equation}\label{eq:ODE y sint}
\begin{split}
    \dot{\boldsymbol{y}}(t) &= F\boldsymbol{y}(t)\,,\\
    \boldsymbol{y}(0) &= \boldsymbol{y}_0\,,
\end{split}
\end{equation}
where $F \in \mathbb{R}^{(d+3)\times(d+3)}$. Worth noting that all state variables in the ODE system \eqref{eq:ODE y sint} are observable.
Then according to Lemma \ref{lemma:identifiability ODE}, the system \eqref{eq:ODE y sint} is $(\boldsymbol{y}_0, F)$-identifiable if and only if condition \textbf{E1} stated in Theorem \ref{thm:identifiability ODE1 sint} is satisfied. That is, under assumption \textbf{E1}, the trajectory $\boldsymbol{y}(\cdot; \boldsymbol{y}_0, F)$ uniquely determines both $\boldsymbol{y}_0$ and matrix $F$. Consequently, it also uniquely determines $(\boldsymbol{x}_0, A, B\boldsymbol{z}_0, B\boldsymbol{v}_1, B\boldsymbol{v}_2)$, thus establishing that the ODE system \eqref{eq:ODE1} is $\boldsymbol{\theta}$-identifiable if and only if assumption \textbf{E1} is satisfied. 
\end{proof}
\newpage
\section{An alternative approach to identifying matrices \textit{B} and \textit{G} in the ODE system (\ref{eq:ODE2})}\label{app:alternative approach}

\subsection{Identifiability condition from \textit{2p} controllable whole trajectories}\label{subsec:iden 2p continous}
Recall that $\boldsymbol{z}_0$ denotes the initial condition of the latent variables in the ODE system \eqref{eq:ODE2}. We further specify the initial condition of the latent variable $z_j$ as $z_{0j}$ for $j=1, \ldots, p$. Assume that it is possible to control the initial condition of each latent variable, $z_{0j}$, independently. Specifically, for each experiment, researchers can intervene in the initial condition of a latent variable, denoted as $z_{0j}^*$. The value of $z_{0j}^*$ is treated as a given value. Under this intervention, the initial conditions of the latent variables are adjusted to $[z_{01}, \ldots, z_{0j}^*, \ldots, z_{0p}]^T$, which we denote as $\tilde{\boldsymbol{z}}_{0j}$. 

To identify matrices $B$ and $G$, it is necessary to have at least two intervened initial conditions for each latent variable, denoted as $z_{0j}^{*1}$ and $z_{0j}^{*2}$ for the latent variable $z_j$. Consequently, the corresponding intervened initial conditions for all latent variables can be represented as $\tilde{\boldsymbol{z}}_{0j}^1$ and $\tilde{\boldsymbol{z}}_{0j}^2$. Under these conditions, we present the definition of the identifiability of the ODE system \eqref{eq:ODE2}.

\begin{definition}\label{def: identifiability ODE2 single intervene}
Given $z_{0j}^{*1}, z_{0j}^{*2} \in \mathbb{R}$ for $j=1, \ldots, p$, for $\boldsymbol{x}_0 \in \mathbb{R}^d, \boldsymbol{z}_0 \in \mathbb{R}^p,  A\in \mathbb{R}^{d\times d}, B\in \mathbb{R}^{d\times p}$ and $G \in \mathbb{R}^{p\times p}$, under the latent DAG assumption, for all $\boldsymbol{x}'_0 \in \mathbb{R}^d$, all $\boldsymbol{z}'_0 \in \mathbb{R}^p$, all $A'\in \mathbb{R}^{d\times d}$, all $B' \in \mathbb{R}^{d\times p}$, and all $G' \in \mathbb{R}^{p\times p}$, we denote $\tilde{\boldsymbol{z}}_{0j}^{i} = [z_{01}, \ldots, z_{0j}^{*i}, \ldots, z_{0p}]^T$ and $(\tilde{\boldsymbol{z}}_{0j}')^{i} = [z_{01}', \ldots, z_{0j}^{*i}, \ldots, z_{0p}']^T$, we further denote $\boldsymbol{\eta}_j^i := (\boldsymbol{x}_0,\tilde{\boldsymbol{z}}_{0j}^{i}, A, B, G) $ and $(\boldsymbol{\eta}'_j)^i := (\boldsymbol{x}'_0,(\tilde{\boldsymbol{z}}_{0j}')^{i}, A', B', G')$ for $i = 1, 2$, we say the ODE system \eqref{eq:ODE2} is $\{\boldsymbol{\eta}_j^{1,2}\}_1^p$-identifiable:
if $(\boldsymbol{x}_0, A, B, G) \neq (\boldsymbol{x}', A', B', G')$, it holds that $\exists i \in \{1,2\}$ and $j \in \{1, \ldots, p\}$ such that $\boldsymbol{x}(\cdot; \boldsymbol{\eta}_j^i)\ $${\neq}$ $\  \boldsymbol{x}(\cdot; (\boldsymbol{\eta}'_j)^i) $.
\end{definition}

Definition \ref{def: identifiability ODE2 single intervene} establishes the identifiability of the ODE system \eqref{eq:ODE2} from $2p$ whole trajectories $\boldsymbol{x}(\cdot;\boldsymbol{\eta}_j^i)$ with $i=1,2$ and $j=1, \ldots, p$. According to this definition, both matrices $B$ and $G$ are identifiable. Based on this definition, we present the identifiability condition.

\begin{theorem}\label{thm: identifiability ODE2 single intervene}
Given $z_{0j}^{*1}, z_{0j}^{*2} \in \mathbb{R}$ with $z_{0j}^{*1}\neq z_{0j}^{*2}$ for $j=1, \ldots, p$, for $\boldsymbol{x}_0 \in \mathbb{R}^d, \boldsymbol{z}_0 \in \mathbb{R}^p, A\in \mathbb{R}^{d\times d}, B\in \mathbb{R}^{d\times p}$ and $G \in \mathbb{R}^{p\times p}$, under the latent DAG assumption, the ODE system \eqref{eq:ODE2} is $\{\boldsymbol{\eta}_j^{1,2}\}_1^p$-identifiable if assumptions $\textbf{B}_5$ and $\textbf{B}_4$ are both satisfied.
\begin{itemize}
    \item [\textbf{B5}:] each $\tilde{\boldsymbol{z}}_{0j}^{i}$ for $i=1, 2$ and $j= 1, \ldots, p$, satisfies assumption \textbf{B1}. That is, if we set $ \boldsymbol{\gamma}_j^i = A^{p}\boldsymbol{x}_0 + \sum_{k=0}^{p-1}A^{p-1-k} B G^k \tilde{\boldsymbol{z}}_{0j}^{i}$, then the set of vectors $\{\boldsymbol{\gamma}_j^i, A\boldsymbol{\gamma}_j^i, \ldots, A^{d-1}\boldsymbol{\gamma}_j^i \}$ is linearly independent for all $i=1, 2$ and $j= 1, \ldots, p$.
\end{itemize}
\end{theorem}
The proof of Theorem \ref{thm: identifiability ODE2 single intervene} is presented below. Assumption $\textbf{B5}$ ensures that the ODE system \eqref{eq:ODE2} is $\boldsymbol{\eta}_j^i$-identifiable for all $i=1, 2$ and $j= 1, \ldots, p$. Consequently, $(\boldsymbol{x}_0, A, B\tilde{\boldsymbol{z}}_{0j}^{i}, BG\tilde{\boldsymbol{z}}_{0j}^{i}, \ldots, BG^{p-1}\tilde{\boldsymbol{z}}_{0j}^{i})$ for all $i=1, 2$ and $j= 1, \ldots, p$ is identifiable. Through straightforward calculations, the identifiability of matrix $B$ is established. To identify matrix $G$, assumption \textbf{B4} is required. 

The assumption that the initial condition of each latent variable $z_i$ can be controlled independently is inspired by the "genetic single-node intervention" proposed in \cite{squires2023linear}, where interventions can be made at each latent node individually. This assumption is relatively more relaxed compared to controlling the initial condition of all latent variables $\boldsymbol{z}$ simultaneously, as discussed in Subsection \ref{subsec:iden p continous}. However, this method requires $p$ more trajectories, totalling $2p$ trajectories, to identify matrices $B$ and $G$.

\begin{proof}
Under assumption \textbf{B5}, since each $\tilde{\boldsymbol{z}}_{0j}^{i}$ satisfies assumption \textbf{B1}. By Theorem \ref{thm: identifiability ODE2 continuous}, the ODE system \eqref{eq:ODE2} is $\boldsymbol{\eta}_j^i$-identifiable for all $i=1, 2$ and $j= 1, \ldots, p$. Consequently, 
\begin{equation*}
    (\boldsymbol{x}_0, A, B\tilde{\boldsymbol{z}}_{0j}^{i}, BG\tilde{\boldsymbol{z}}_{0j}^{i}, \ldots, BG^{p-1}\tilde{\boldsymbol{z}}_{0j}^{i})
\end{equation*}
for all $i=1, 2$ and $j= 1, \ldots, p$ is identifiable.

We express $B\tilde{\boldsymbol{z}}_{0j}^{i}$ as 
\begin{equation*}
    B\tilde{\boldsymbol{z}}_{0j}^{i} =  \begin{bmatrix}
        B_{11} & \ldots & B_{1j} & \ldots & B_{1p}\\
        \vdots & \ddots & \vdots & \ddots & \vdots \\
        B_{d1} & \ldots & B_{dj} & \ldots & B_{dp}
    \end{bmatrix} \begin{bmatrix}
        z_{01}\\
        \vdots\\
        z_{0j}^{*i}\\
        \vdots\\
        z_{0p}
    \end{bmatrix}\,.
\end{equation*}
We know that $B\tilde{\boldsymbol{z}}_{0j}^{i} \in \mathbb{R}^{d}$ is identifiable for $i = 1, 2$. Thus, the first entry of $B\tilde{\boldsymbol{z}}_{0j}^{i}$, denoted as $(B\tilde{\boldsymbol{z}}_{0j}^{i})_1$, is identifiable and can be expressed as
\begin{equation*}
\begin{split}
    (B\tilde{\boldsymbol{z}}_{0j}^{1})_1 &= B_{11}z_{01} + \ldots + B_{1j}z_{0j}^{*1} +\ldots + B_{1p}z_{0p}\\
     (B\tilde{\boldsymbol{z}}_{0j}^{2})_1 &= B_{11}z_{01} + \ldots + B_{1j}z_{0j}^{*2} +\ldots + B_{1p}z_{0p}
\end{split}\,.
\end{equation*}
Since $z_{0j}^{*1}$ and $ z_{0j}^{*2}$ are given values, we can easily calculate the value of $B_{1j}$. Similarly, one can calculate the values of $B_{mj}$ for all $m=1, \ldots, d$ and $j=1, \ldots, p$, thereby establishing the identifiability of matrix $B$.

In a similar manner, matrices $BG, BG^2, \ldots, BG^{p-1}$ are also identifiable. Then, according to the proof \ref{proof: ODE2 more} of Theorem \ref{thm: identifiability ODE2 more}, the matrix $G$ is identifiable under assumption \textbf{B4}.
\end{proof}

\subsection{Identifiability condition from discrete observations sampled from \textit{2p} controllable trajectories}
We further extend the identifiability analysis of the ODE system \eqref{eq:ODE2} to cases where only discrete observations from $2p$ controllable trajectories are available.

\begin{definition}\label{def: identifiability ODE2 discrete single intervene}
Given $z_{0j}^{*1}, z_{0j}^{*2} \in \mathbb{R}$ for $j=1, \ldots, p$, for $\boldsymbol{x}_0 \in \mathbb{R}^d, \boldsymbol{z}_0 \in \mathbb{R}^p, A\in \repair{\mathcal A_\Delta}, B\in \mathbb{R}^{d\times p}$ and $G \in \mathbb{R}^{p\times p}$. For any $n \geqslant 1$, let $t_k, k = 1, \ldots, n$ be any $n$ time points and $\boldsymbol{x}_{jk}^i : = \boldsymbol{x}(t_k; \boldsymbol{\eta}_j^i)$ be the error-free observation of the trajectory $\boldsymbol{x}(\cdot; \boldsymbol{\eta}_j^i)$ at time $t_k$. Under the latent DAG assumption, we say the ODE system \eqref{eq:ODE2} is $\{\boldsymbol{\eta}_j^{1,2}\}_1^p$-identifiable from $\boldsymbol{x}_{j1}^i, \ldots, \boldsymbol{x}_{jn}^i$, $i = 1, 2$ and $j=1, \ldots, p$, if for all $\boldsymbol{x}'_0 \in \mathbb{R}^d$, all $\boldsymbol{z}'_0 \in \mathbb{R}^p$, all $A'\in \repair{\mathcal A_\Delta}$, all $B' \in \mathbb{R}^{d\times p}$, and all $G' \in \mathbb{R}^{p\times p}$ with $(\boldsymbol{x}_0, A, B, G) \neq (\boldsymbol{x}'_0, A', B', G')$, it holds that $\exists i \in \{1,2\}, j \in \{1, \ldots, p\}$ and $k \in \{1, \ldots, n\}$ such that $\boldsymbol{x}(t_k; \boldsymbol{\eta}_j^i){\neq} \boldsymbol{x}(t_k; (\boldsymbol{\eta}'_j)^i)$.
\end{definition}
Based on Definition \ref{def: identifiability ODE2 discrete single intervene} we present the identifiability condition.
\begin{theorem}\label{thm: identifiability ODE2 discrete single intervene}
Given $z_{0j}^{*1}, z_{0j}^{*2} \in \mathbb{R}$ with $z_{0j}^{*1}\neq z_{0j}^{*2}$ for $j=1, \ldots, p$, for $\boldsymbol{x}_0 \in \mathbb{R}^d, \boldsymbol{z}_0 \in \mathbb{R}^p,  A\in \repair{\mathcal A_\Delta}, B\in \mathbb{R}^{d\times p}$ and $G \in \mathbb{R}^{p\times p}$. We define new observation $\boldsymbol{y}_{jk}^i := [(\boldsymbol{x}_{jk}^i)^T, 1, t_k, t_k^2,\ldots, t_k^{p-1}]^T\in \mathbb{R}^{d+p}$, for $i =1, 2, j = 1, \ldots, p$ and $k = 1, \ldots, n$. Under the latent DAG assumption \repair{and \textbf{C0}}, the ODE system \eqref{eq:ODE2} is $\{\boldsymbol{\eta}_j^{1,2}\}_1^p$-identifiable from discrete observations $\boldsymbol{x}_{j1}^i, \ldots, \boldsymbol{x}_{jn}^i$, $i = 1, 2$ and $j=1, \ldots, p$, if assumptions \textbf{C3} and \textbf{B4} are both satisfied.
\begin{itemize}
    \item [\textbf{C3}:] for each $i \in \{1,2\}, j \in \{1, \ldots, p\}$ there exists $(d+p)$ $\boldsymbol{y}_{jk}^i$'s with indexes denoting as $\{k_{j1}^i, k_{j2}^i, \ldots, k_{j,d+p}^i\} \subseteq \{1, 2, \ldots, n\}$, such that the set of vectors $\{\boldsymbol{y}_{jk_{j1}^i}^i, \boldsymbol{y}_{jk_{j2}^i}^i, \ldots, \boldsymbol{y}_{jk_{j,d+p}^i}^i\}$ is linearly independent.
\end{itemize}
\end{theorem}

The proof of Theorem \ref{thm: identifiability ODE2 discrete single intervene} is presented below. Assumption \textbf{C3} ensures that the ODE system \eqref{eq:ODE2} is $\boldsymbol{\eta}_j^i$-identifiable from discrete observations $\boldsymbol{x}_{j1}^i, \ldots, \boldsymbol{x}_{jn}^i$  for all $i=1, 2$ and $j= 1, \ldots, p$. As in Subsection \ref{subsec:iden 2p continous}, matrix $B$ is identifiable. Then, under assumption \textbf{B4}, matrix $G$ is also identifiable.
\begin{proof}
    Under assumption \textbf{C3}, for each  $i \in \{1,2\}$ and $j \in \{1, \ldots, p\}$, the corresponding observations satisfy assumption \textbf{C1}. Based on Theorem \ref{thm: identifiability ODE2 discrete}, the ODE system \eqref{eq:ODE2} is $\boldsymbol{\eta}_j^i$-identifiable for all $i=1, 2$ and $j= 1, \ldots, p$. Consequently, 
\begin{equation*}
    (\boldsymbol{x}_0, A, B\tilde{\boldsymbol{z}}_{0j}^{i}, BG\tilde{\boldsymbol{z}}_{0j}^{i}, \ldots, BG^{p-1}\tilde{\boldsymbol{z}}_{0j}^{i})
\end{equation*}
for all $i=1, 2$ and $j= 1, \ldots, p$ is identifiable.

Following the proof of Theorem \ref{thm: identifiability ODE2 single intervene}, matrix $B$ is identifiable. Under assumption \textbf{B4}, matrix $G$ is also identifiable.
\end{proof}
\newpage
\section{More simulation results}\label{app:more simulations}
In this section, we present additional simulation results for higher-dimensional cases, along with simulations that incorporate a variety of ground-truth parameter configurations.

\subsection{Higher dimensional cases}
In this subsection, for the $\boldsymbol{\eta}$-(un)identifiable cases of the ODE system \eqref{eq:ODE2}, we provide a case with $d=5$ and $p=5$. The true underlying parameters of the systems are provided below. Initial parameter values are set to the true parameters plus a random value drawn from a uniform distribution $U(-0.14, 0.14)$ for each replication. To ensure reliability in the estimation results, we perform 50 independent random replications for each configuration, reporting the mean and variance of the squared error in Table \ref{table-fixparam-single-d5p5}.
\begin{equation*}
\begin{split}
      &A = \begin{bmatrix}
        2 & -2 & 1 & 1 & 1\\
        -1 & 1 & 0 & 2 & -2\\
        -2 & 2 & 0 & -1 & -2 \\
        -1 & -1 & -2 & -1 & 2 \\
        1 & -2 & 1 & -2 & 0
    \end{bmatrix}\,, \ \ 
    B = \begin{bmatrix}
        1 & -2 & -1 & 1 & 1\\
        1 & -2 & -1 & -1 & -1\\
        -2 & 0 & 2 & 1 & 1\\
        0 & 2 & 0 & -2 & -2\\
        2 & -2 & 2 & -1 & 2
    \end{bmatrix}\,, \\
    &G = \begin{bmatrix}
        0 & 0 & 0 & -2 & -1\\
        0 & 0 & -1 & 1 & 1\\
        0 & 0 & 0 & 1 & 2\\
        0 & 0 & 0 & 0 & 2\\
        0 & 0 & 0 & 0 & 0
    \end{bmatrix}\,, \ \
    A' = \boldsymbol{I}_5\,, \ \  
     \boldsymbol{x}_0 = \begin{bmatrix}
        2\\
        -2\\
        2\\
        1\\
        0
    \end{bmatrix}\,, \ \ \
    \boldsymbol{z}_0 = \begin{bmatrix}
        -2\\
        -1\\
        -1\\
        1\\
        -2
    \end{bmatrix}\,, \\
    & \boldsymbol{\eta}\mbox{-identifiable: } \boldsymbol{\eta} = (\boldsymbol{x}_0, \boldsymbol{z}_0, A, B, G), \mbox{ unidentifiable: } \boldsymbol{\eta} = (\boldsymbol{x}_0, \boldsymbol{z}_0, A', B, G)\,.\\
\end{split}
\end{equation*}
$\boldsymbol{I}_j$ denotes a $j\times j$ identity matrix.

\begin{table}[htbp]
    \caption{MSEs of the $\boldsymbol{\eta}$-(un)identifiable cases of the ODE (3) with $d=5, p=5$}
    \label{table-fixparam-single-d5p5}
    \centering
    \begin{tabu}{cccccccc}
    \toprule
     & $\boldsymbol{n}$ & $A$  & $B\boldsymbol{z}_0$  & $BG\boldsymbol{z}_0$  & $BG^2\boldsymbol{z}_0$ &  $BG^3\boldsymbol{z}_0$ & $BG^4\boldsymbol{z}_0$ \\
    \midrule
    \addlinespace
    \multirow{6}{*}[-2ex]{\rotatebox{90}{\textbf{Identifiable}}} &\multirow{2}{*}{10}  & 0.0148 & 0.3911 & 0.9624 & 0.7316  & 0.1037 & 0.0096  \\
    \rowfont{\footnotesize}  & & ($\pm$0.0006) & ($\pm$0.5989) & ($\pm$3.9249) & ($\pm$1.8971) & ($\pm$0.0374) & ($\pm$0.0003) \\
    \addlinespace
    &\multirow{2}{*}{100}  & 0.0059 & 0.1529  & 0.1726 & 0.2447 & 0.0212 & 0.0012  \\
    \rowfont{\footnotesize}  & & ($\pm$4.01\text{E-}05) & ($\pm$0.0277)  & ($\pm$0.0541) & ($\pm$0.0748) & ($\pm$0.0007) & ($\pm$1.10\text{E-}05)  \\
    \addlinespace
    &\multirow{2}{*}{1000}  & 0.0053 & 0.1394  & 0.1241 & 0.2119 & 0.0164 & 0.0004 \\
    \rowfont{\footnotesize}  & & ($\pm$2.92\text{E-}05) & ($\pm$0.0200)  & ($\pm$0.0251) & ($\pm$0.0479) & ($\pm$0.0004) & ($\pm$6.00\text{E-}07)  \\
    \addlinespace
    \midrule
    \addlinespace
    \multirow{6}{*}[-2ex]{\rotatebox{90}{\textbf{Unidentifiable}}} &\multirow{2}{*}{10}  & 0.0853 & 1.0067  & 3.7422 & 2.7696  & 0.9229 & 0.0508  \\
    \rowfont{\footnotesize}   && ($\pm$0.0075) & ($\pm$1.3518) & ($\pm$55.8402) & ($\pm$24.5043) & ($\pm$2.7959) & ($\pm$0.0111) \\
    \addlinespace
    &\multirow{2}{*}{100}  & 0.0357 & 0.4091  & 1.0428 & 0.9782 & 0.3871 & 0.0256  \\
    \rowfont{\footnotesize}   && ($\pm$0.0019) & ($\pm$0.3812)  & ($\pm$2.1792) & ($\pm$5.3654) & ($\pm$0.6747) & ($\pm$0.0032)  \\
    \addlinespace
    &\multirow{2}{*}{1000}  & 0.0332 & 0.3286  & 0.7123 & 0.9782 & 0.5487 & 0.0393  \\
    \rowfont{\footnotesize}   && ($\pm$0.0017) & ($\pm$0.1824)  & ($\pm$1.8836) & ($\pm$2.3163) & ($\pm$0.9240) & ($\pm$0.0047)  \\
    \addlinespace
    \bottomrule
    \end{tabu}
\end{table}

For $\{\boldsymbol{\eta}_i\}_1^p$-(un)identifiable cases of the ODE system \eqref{eq:ODE2}, we consider a case with $d=10$ and $p=5$. To accelerate estimation, sparsity is introduced in the parameter matrices by randomly setting 70, 35, and 20 entries in matrices $A$, $B$ and $G$, respectively, as zero. The true underlying parameters of the systems are provided below. Initial parameter values are set to the true parameters plus a random value drawn from a uniform distribution $U(-0.1, 0.1)$ for each replication. To ensure reliability in the estimation results, we perform 50 independent random replications for each configuration, reporting the mean and variance of the squared error in Table \ref{table-fixparam-multiple-d10p5}.
\begin{equation*}
\begin{split}
      &A = \begin{bmatrix}
        0 & 0 & -2 & -1 & 1 & 2 & 0 & -2 & -1 & 0\\
        0 & 0 & 0 & 0 & 0 & 2 & 0 & -2 & 2 & 0\\
        0 & 0 & 0 & 0 & 0 & 2 & 0 & 1 & 1 & 0\\
        0 & 0 & 0 & -1 & 0 & 0 & 1 & 0 & -2 & 0\\
        2 & 0 & 0 & -1 & 0 & -2 & 0 & 0 & -1 & 1\\
        2 & 0 & 0 & 0 & 0 & 0 & 0 & 2 & 0 & -2\\
        0 & 2 & 0 & 0 & 0 & 0 & 0 & 0 & 0 & 0\\
        -2 & -1 & 0 & 0 & 0 & 0 & 0 & 0 & 0 & 0\\
        0 & 0 & 0 & -2 & 0 & 0 & 0 & 0 & 0 & -2\\
        0 & 0 & 0 & 0 & -1 & 0 & 0 & 0 & 0 & -1
    \end{bmatrix}\,, \ \ 
    B = \begin{bmatrix}
        -1 & 0 & 0 & 0 & 2\\
        0 & -1 & 0 & 2 & 0\\
        0 & -1 & 0 & 0 & 0\\
        0 & 0 & 0 & 1 & 1\\
        0 & 0 & 0 & 0 & 0\\
        0 & 1 & 0 & 0 & 1\\
        0 & 0 & -1 & 0 & 0 \\
        1 & 0 & 0 & 0 & 0 \\
        1 & 0 & 0 & 0 & -1\\
        -1 & 0 & 0 & 0 & -1
    \end{bmatrix}\,, \\
    &G = \begin{bmatrix}
        0 & 1 & -1 & 0 & 2\\
        0 & 0 & 2 & 0 & 0\\
        0 & 0 & 0 & -1 & 0\\
        0 & 0 & 0 & 0 & 0\\
        0 & 0 & 0 & 0 & 0
    \end{bmatrix}\,, \ \
    A' = \boldsymbol{I}_{10}\,, \\ 
     &\boldsymbol{x}_0 = \begin{bmatrix}
        -2 & 0 & 0 & -2 & 2 & -1 & 1 & 0 & 1 & 1
    \end{bmatrix}^{\top}\,, \ \ \
    \boldsymbol{z}_0^{*i} = \boldsymbol{e}_i\,, \mbox{for } i =1, \ldots, 5\,.\\
    & \{\boldsymbol{\eta}_i\}_1^p\mbox{-identifiable: } \boldsymbol{\eta}_i = (\boldsymbol{x}_0, \boldsymbol{z}_0^{*i}, A, B, G), \mbox{ unidentifiable: } \boldsymbol{\eta}_i = (\boldsymbol{x}_0, \boldsymbol{z}_0^{*i}, A', B, G)\,. \\
\end{split}
\end{equation*}
$\boldsymbol{e}_i$ stands for a $p$-dimensional vector, with the $i$-th entry being $1$ and the other entries being $0$.

\begin{table}[htbp]
    \caption{MSEs of the  $\{\boldsymbol{\eta}_i\}_1^p$-(un)identifiable cases of the ODE (3) with $d=10, p=5$}
    \label{table-fixparam-multiple-d10p5}
    \centering
    \begin{tabu}{ccccccc}
    \toprule
    \multicolumn{1}{c}{\multirow{2}{*}{$\boldsymbol{n}$}} & \multicolumn{3}{c}{\textbf{Identifiable}} & \multicolumn{3}{c}{\textbf{Unidentifiable}}\\
    \cmidrule(lr){2-4} \cmidrule(l){5-7}
       & $A$  & $B$ & $G$ &  $A$ &  $B$  & $G$ \\
    \midrule
    \addlinespace
    \multirow{2}{*}{10}   & 1.53\text{E-}11 & 2.49\text{E-}10  & 3.01\text{E-}10 & 0.8345 & 0.2118 & 0.0037 \\
    \rowfont{\footnotesize}   & ($\pm$2.36\text{E-}21) & ($\pm$6.30\text{E-}19) & ($\pm$9.20\text{E-}19) & ($\pm$0.6268)  & ($\pm$0.0260) & ($\pm$0.0002)  \\
    \addlinespace
    \multirow{2}{*}{30}   & 9.15\text{E-}13 & 1.49\text{E-}11  & 1.80\text{E-}11 & 0.7216 & 0.1952 & 1.25\text{E-}21 \\
    \rowfont{\footnotesize}   & ($\pm$4.45\text{E-}24) & ($\pm$1.18\text{E-}21) & ($\pm$1.73\text{E-}21) & ($\pm$0.4099)  & ($\pm$0.0156) & ($\pm$5.18\text{E-}41)  \\
    \addlinespace
    \multirow{2}{*}{50}   & 9.64\text{E-}14 & 1.57\text{E-}12  & 1.90\text{E-}12 & 0.6510 & 0.2211 & 0.0042 \\
    \rowfont{\footnotesize}   & ($\pm$1.29\text{E-}25) & ($\pm$3.43\text{E-}23) & ($\pm$5.02\text{E-}23) & ($\pm$0.2251)  & ($\pm$0.0278) & ($\pm$0.0003)  \\
    \addlinespace
    \bottomrule
    \end{tabu}
\end{table}

Tables \ref{table-fixparam-single-d5p5} and \ref{table-fixparam-multiple-d10p5} present results similar to those in Tables \ref{table1} and \ref{table2}, providing strong empirical support for the validity of our proposed identifiability conditions.

\subsection{Various true parameters}
To further support our proposed identifiability conditions, we conduct additional simulations incorporating a variety of ground-truth parameter configurations, rather than a fixed underlying parameter set. Specifically, for each simulation run, a unique ground-truth parameter configuration was generated using different random seeds, and we subsequently reported the mean and variance of the squared error across all results. For the low-dimensional $\boldsymbol{\eta}$ and $\{\boldsymbol{\eta}_i\}_1^p$ (un)identifiable cases, we perform 100 replications, while for the higher-dimensional cases, we perform 50 replications. Additionally, in the $\{\boldsymbol{\eta}_i\}_1^p$-(un)identifiable cases, we initialize the parameter values as the true parameters plus a random value drawn from $U(-0.1, 0.1)$ for the $d=3, p=3$ case and from $U(-0.05, 0.05)$ for the $d=10, p=5$ cases. For the $\boldsymbol{\eta}$-(un)identifiable cases, the initialization settings are the same as those used in the fixed-parameter configurations.

The simulation results are presented in Tables \ref{table-varparams-single-d3p3}, \ref{table-varparams-multiple-d3p3}, \ref{table-varparams-single-d5p5}, and \ref{table-varparams-multiple-d10p5}. Across all these tables, parameter estimates in the identifiable cases are notably more accurate than in the unidentifiable cases, providing strong empirical support for the validity of our proposed identifiability conditions.

 It is noteworthy, however, that even in theoretically identifiable cases, certain scenarios emerge where parameter identification is challenging in practice; we refer to these as hard estimate cases. In these instances, estimates may deviate significantly from satisfactory values, similar to challenges encountered in fully observable ODE systems \eqref{eq:ODE} as discussed in \cite{qiu2022identifiability}. Consequently, for identifiable cases with varying true parameter configurations, the results are less precise than those for corresponding fixed-parameter cases, due to the inclusion of some hard estimate instances. Investigating the practical identifiability of the ODE system \eqref{eq:ODE2} remains an intriguing direction for future research.

\begin{table}[htbp]
    \caption{MSEs of the $\boldsymbol{\eta}$-(un)identifiable cases of the ODE \eqref{eq:ODE2} - with various true parameters}
    \label{table-varparams-single-d3p3}
    \centering
    \begin{tabu}{ccccccccc}
    \toprule
    \multicolumn{1}{c}{\multirow{2}{*}{$\boldsymbol{n}$}} & \multicolumn{4}{c}{\textbf{Identifiable}} & \multicolumn{4}{c}{\textbf{Unidentifiable}}\\
    \cmidrule(lr){2-5} \cmidrule(l){6-9}
       & $A$  & $B\boldsymbol{z}_0$  & $BG\boldsymbol{z}_0$  & $BG^2\boldsymbol{z}_0$ & $A$ &  $B\boldsymbol{z}_0$  & $BG\boldsymbol{z}_0$  & $BG^2\boldsymbol{z}_0$\\
    \midrule
    \addlinespace
    \multirow{2}{*}{10}   & 0.0060 & 0.0157  & 0.1698 & 0.2297 & 0.0691 & 0.2720 & 1.3133 & 0.6622 \\
    \rowfont{\scriptsize}   & ($\pm$0.0008) & ($\pm$0.0036) & ($\pm$0.5665) & ($\pm$1.1053)  & ($\pm$0.0203) & ($\pm$0.5914) & ($\pm$7.4471) & ($\pm$8.5348) \\
    \addlinespace
    \addlinespace
    \multirow{2}{*}{100}   & 0.0026 & 0.0108  & 0.0820 & 0.1287 & 0.0283 & 0.1003 & 0.4880 & 0.2649 \\
    \rowfont{\scriptsize}   & ($\pm$9.27\text{E-}05) & ($\pm$0.0022) & ($\pm$0.1159) & ($\pm$0.7042)  & ($\pm$0.0031) & ($\pm$0.0441) & ($\pm$2.6547) & ($\pm$1.6631) \\
    \addlinespace
    \addlinespace
    \multirow{2}{*}{500}   & 0.0020 & 0.0092  & 0.0870 & 0.0705 & 0.0227 & 0.1061 & 0.5015 & 0.2574 \\
    \rowfont{\scriptsize}   & ($\pm$6.48\text{E-}05) & ($\pm$0.0023) & ($\pm$0.1941) & ($\pm$0.1179)  & ($\pm$0.0018) & ($\pm$0.0672) & ($\pm$3.0811) & ($\pm$2.0779) \\
    \addlinespace
    \bottomrule
    \end{tabu}
\end{table}

\begin{table}[htbp]
    \caption{MSEs of the  $\{\boldsymbol{\eta}_i\}_1^p$-(un)identifiable cases of the ODE \eqref{eq:ODE2} - with various true parameters}
    \label{table-varparams-multiple-d3p3}
    \centering
    \begin{tabu}{ccccccc}
    \toprule
    \multicolumn{1}{c}{\multirow{2}{*}{$\boldsymbol{n}$}} & \multicolumn{3}{c}{\textbf{Identifiable}} & \multicolumn{3}{c}{\textbf{Unidentifiable}}\\
    \cmidrule(lr){2-4} \cmidrule(l){5-7}
       & $A$  & $B$ & $G$ &  $A$ &  $B$  & $G$ \\
    \midrule
    \addlinespace
    \multirow{2}{*}{10}   & 0.0006 & 1.89\text{E-}5  &  0.0009 & 0.0861 & 0.0088 & 0.0101 \\
    \rowfont{\footnotesize}   & ($\pm$2.21\text{E-}5) & ($\pm$3.55\text{E-}8) & ($\pm$6.71\text{E-}5) & ($\pm$0.1773)  & ($\pm$0.0020) & ($\pm$0.0045)  \\
    \addlinespace
    \addlinespace
    \multirow{2}{*}{30}   & 0.0006 & 1.87\text{E-}5  &  0.0010 & 0.0789 & 0.0092 & 0.0104 \\
    \rowfont{\footnotesize}   & ($\pm$2.20\text{E-}5) & ($\pm$3.47\text{E-}8) & ($\pm$6.64\text{E-}5) & ($\pm$0.1280)  & ($\pm$0.0028) & ($\pm$0.0046)  \\
    \addlinespace
    \addlinespace
    \multirow{2}{*}{50}   & 0.0006 & 1.88\text{E-}5 & 0.0009  & 0.0503 & 0.0063  & 0.0114  \\
    \rowfont{\footnotesize}   & ($\pm$2.21\text{E-}5) & ($\pm$3.51\text{E-}8) & ($\pm$6.67\text{E-}5) & ($\pm$0.0430)  & ($\pm$0.0006) & ($\pm$0.0047)  \\
    \addlinespace
    \bottomrule
    \end{tabu}
\end{table}

\begin{table}[htbp]
    \caption{MSEs of the $\boldsymbol{\eta}$-(un)identifiable cases of the ODE (3) with $d=5, p=5$ - with various true parameters}
    \label{table-varparams-single-d5p5}
    \centering
    \begin{tabu}{cccccccc}
    \toprule
     & $\boldsymbol{n}$ & $A$  & $B\boldsymbol{z}_0$  & $BG\boldsymbol{z}_0$  & $BG^2\boldsymbol{z}_0$ &  $BG^3\boldsymbol{z}_0$ & $BG^4\boldsymbol{z}_0$ \\
    \midrule
    \addlinespace
    \multirow{6}{*}[-2ex]{\rotatebox{90}{\textbf{Identifiable}}} &\multirow{2}{*}{10}  & 0.0144 & 0.1215 & 1.4643 & 2.1890  & 1.8254 & 0.4826  \\
    \rowfont{\footnotesize}  & & ($\pm$0.0004) & ($\pm$0.0757) & ($\pm$8.3976) & ($\pm$54.9706) & ($\pm$48.7033) & ($\pm$5.7127) \\
    \addlinespace
    &\multirow{2}{*}{100}  & 0.0041 & 0.0395  & 0.2850 & 0.3891 & 0.2078 & 0.0239  \\
    \rowfont{\footnotesize}  & & ($\pm$4.55\text{E-}05) & ($\pm$0.0092)  & ($\pm$0.1739) & ($\pm$0.4936) & ($\pm$0.2950) & ($\pm$0.0024)  \\
    \addlinespace
    &\multirow{2}{*}{1000}  & 0.0032 & 0.0337  & 0.1934 & 0.2242 & 0.1197 & 0.0181 \\
    \rowfont{\footnotesize}  & & ($\pm$3.26\text{E-}05) & ($\pm$0.0049)  & ($\pm$0.0686) & ($\pm$0.2180) & ($\pm$0.0712) & ($\pm$0.0014)  \\
    \addlinespace
    \midrule
    \addlinespace
    \multirow{6}{*}[-2ex]{\rotatebox{90}{\textbf{Unidentifiable}}} &\multirow{2}{*}{10}  & 0.0740 & 0.4599  & 2.8628 & 1.8743  & 0.4834 & 0.0334  \\
    \rowfont{\footnotesize}   && ($\pm$0.0047) & ($\pm$0.4841) & ($\pm$9.5476) & ($\pm$8.6653) & ($\pm$1.2606) & ($\pm$0.0147) \\
    \addlinespace
    &\multirow{2}{*}{100}  & 0.0263 & 0.2142  & 1.1678 & 1.2354 & 0.2878 & 0.0193  \\
    \rowfont{\footnotesize}   && ($\pm$0.0031) & ($\pm$0.1869)  & ($\pm$8.2277) & ($\pm$9.4970) & ($\pm$0.8655) & ($\pm$0.0052)  \\
    \addlinespace
    &\multirow{2}{*}{1000}  & 0.0142 & 0.1389  & 0.6979 & 0.6701 & 0.0732 & 0.0062  \\
    \rowfont{\footnotesize}   && ($\pm$0.0003) & ($\pm$0.0463)  & ($\pm$1.2080) & ($\pm$1.5228) & ($\pm$0.0336) & ($\pm$0.0003)  \\
    \addlinespace
    \bottomrule
    \end{tabu}
\end{table}

\begin{table}[ht]
    \caption{MSEs of the  $\{\boldsymbol{\eta}_i\}_1^p$-(un)identifiable cases of the ODE (3) with $d=10, p=5$ - with various true parameters}
    \label{table-varparams-multiple-d10p5}
    \centering
    \begin{tabu}{ccccccc}
    \toprule
    \multicolumn{1}{c}{\multirow{2}{*}{$\boldsymbol{n}$}} & \multicolumn{3}{c}{\textbf{Identifiable}} & \multicolumn{3}{c}{\textbf{Unidentifiable}}\\
    \cmidrule(lr){2-4} \cmidrule(l){5-7}
       & $A$  & $B$ & $G$ &  $A$ &  $B$  & $G$ \\
    \midrule
    \addlinespace
    \multirow{2}{*}{10}   & 0.0044 & 0.0350  & 0.0287 & 0.6266 & 0.1310 & 0.0054 \\
    \rowfont{\footnotesize}   & ($\pm$0.0001) & ($\pm$0.0098) & ($\pm$0.0053) & ($\pm$0.1524)  & ($\pm$0.0269) & ($\pm$0.0004)  \\
    \addlinespace
    \multirow{2}{*}{30}   & 0.0067 & 0.1258  & 0.0315 & 0.5833 & 0.1058 & 0.0021 \\
    \rowfont{\footnotesize}   & ($\pm$0.0005) & ($\pm$0.5097) & ($\pm$0.0104) & ($\pm$0.2085)  & ($\pm$0.0114) & ($\pm$8.79\text{E-}05)  \\
    \addlinespace
    \multirow{2}{*}{50}   & 0.0033 & 0.0323  & 0.0354 & 0.5193 & 0.1108 & 0.0021 \\
    \rowfont{\footnotesize}   & ($\pm$5.66\text{E-}05) & ($\pm$0.0103) & ($\pm$0.0084) & ($\pm$0.0982)  & ($\pm$0.0146) & ($\pm$9.02\text{E-}05)  \\
    \addlinespace
    \bottomrule
    \end{tabu}
\end{table}


\clearpage
\section*{NeurIPS Paper Checklist}

\begin{enumerate}

\item {\bf Claims}
    \item[] Question: Do the main claims made in the abstract and introduction accurately reflect the contributions and scope of our paper?
    \item[] Answer: \answerYes{}
    \item[] Justification: The main claims presented in our abstract and introduction accurately reflect our paper's contributions and scope.

\item {\bf Limitations}
    \item[] Question: Does the paper discuss the limitations of the work performed by the authors?
    \item[] Answer: \answerYes{}
    \item[] Justification: The limitations of our work are discussed in Section \ref{sec:conclusion}.

\item {\bf Theory Assumptions and Proofs}
    \item[] Question: For each theoretical result, does the paper provide the full set of assumptions and a complete (and correct) proof?
    \item[] Answer: \answerYes{}
    \item[] Justification: We provide the full set of assumptions and a complete proof for each theoretical result presented in our paper.

    \item {\bf Experimental Result Reproducibility}
    \item[] Question: Does the paper fully disclose all the information needed to reproduce the main experimental results of the paper to the extent that it affects the main claims and/or conclusions of the paper (regardless of whether the code and data are provided or not)?
    \item[] Answer: \answerYes{}
    \item[] Justification: We provide all experimental details in Section \ref{sec:simulations} and include the code in the supplemental material.

\item {\bf Open access to data and code}
    \item[] Question: Does the paper provide open access to the data and code, with sufficient instructions to faithfully reproduce the main experimental results, as described in supplemental material?
    \item[] Answer: \answerYes{}
    \item[] Justification: We provide our code in the supplemental material.

\item {\bf Experimental Setting/Details}
    \item[] Question: Does the paper specify all the training and test details (e.g., data splits, hyperparameters, how they were chosen, type of optimizer, etc.) necessary to understand the results?
    \item[] Answer: \answerYes{}
    \item[] Justification: We provide all experimental details in Section \ref{sec:simulations}, and it is noteworthy that our experiments do not require any training phase.

\item {\bf Experiment Statistical Significance}
    \item[] Question: Does the paper report error bars suitably and correctly defined or other appropriate information about the statistical significance of the experiments?
    \item[] Answer: \answerYes{}
    \item[] Justification: We conduct 100 independent random replications for each configuration and report the mean and variance of the squared error.

\item {\bf Experiments Compute Resources}
    \item[] Question: For each experiment, does the paper provide sufficient information on the computer resources (type of compute workers, memory, time of execution) needed to reproduce the experiments?
    \item[] Answer: \answerNo{}
    \item[] Justification: Since our experiments solely consist of simulations designed to validate our theoretical findings, the computational resources employed are not a consideration for our research objectives.
    
\item {\bf Code Of Ethics}
    \item[] Question: Does the research conducted in the paper conform, in every respect, with the NeurIPS Code of Ethics \url{https://neurips.cc/public/EthicsGuidelines}?
    \item[] Answer: \answerYes{}
    \item[] Justification: Our research conducted in the paper conform, in every respect, with the NeurIPS Code of Ethics.

\item {\bf Broader Impacts}
    \item[] Question: Does the paper discuss both potential positive societal impacts and negative societal impacts of the work performed?
    \item[] Answer: \answerNA{}.
    \item[] Justification: We claim that this work does not present any foreseeable positive or negative social impact.
    
\item {\bf Safeguards}
    \item[] Question: Does the paper describe safeguards that have been put in place for responsible release of data or models that have a high risk for misuse (e.g., pretrained language models, image generators, or scraped datasets)?
    \item[] Answer: \answerNA{}.
    \item[] Justification: This paper does not require safeguards.

\item {\bf Licenses for existing assets}
    \item[] Question: Are the creators or original owners of assets (e.g., code, data, models), used in the paper, properly credited and are the license and terms of use explicitly mentioned and properly respected?
    \item[] Answer: \answerNA{}.
    \item[] Justification: This paper does not use existing assets.

\item {\bf New Assets}
    \item[] Question: Are new assets introduced in the paper well documented and is the documentation provided alongside the assets?
    \item[] Answer: \answerNA{}
    \item[] Justification: This paper does not release new assets.

\item {\bf Crowdsourcing and Research with Human Subjects}
    \item[] Question: For crowdsourcing experiments and research with human subjects, does the paper include the full text of instructions given to participants and screenshots, if applicable, as well as details about compensation (if any)? 
    \item[] Answer: \answerNA{}.
    \item[] Justification: We did not use crowdsourcing or conduct research with human subjects.

\item {\bf Institutional Review Board (IRB) Approvals or Equivalent for Research with Human Subjects}
    \item[] Question: Does the paper describe potential risks incurred by study participants, whether such risks were disclosed to the subjects, and whether Institutional Review Board (IRB) approvals (or an equivalent approval/review based on the requirements of your country or institution) were obtained?
    \item[] Answer: \answerNA{}.
    \item[] Justification:  We did not use crowdsourcing or conduct research with human subjects.

\end{enumerate}

\end{document}